\ifdefined\pdftexversion\pdfoutput=1\fi 
\documentclass{article}
\usepackage[T1]{fontenc}
\usepackage{conference,times}
\conffinalcopy

\usepackage{amsmath,amssymb,amsthm}
\usepackage{booktabs}
\usepackage{multirow}
\usepackage{tabularx}
\usepackage{graphicx}
\usepackage{flafter}
\usepackage{placeins}
\usepackage{listings}
\usepackage{algorithm}
\usepackage{algpseudocode}
\usepackage{tikz}
\usetikzlibrary{positioning,arrows.meta,fit,calc}
\usepackage{xcolor}
\usepackage[most]{tcolorbox}
\usepackage[hidelinks]{hyperref}
\usepackage{url}

\definecolor{achieve}{RGB}{31,119,120}
\definecolor{refute}{RGB}{178,58,50}
\definecolor{bdh}{RGB}{110,110,118}

\newcommand{\algdelta}[1]{\begingroup\setlength{\fboxsep}{0.65pt}\colorbox{achieve!13}{#1}\endgroup}
\newcommand{\algside}[1]{\Comment{\parbox[t]{.34\linewidth}{\raggedright\footnotesize #1}}}

\newtheorem{theorem}{Theorem}
\newtheorem{lemma}{Lemma}
\newcommand{\F}{\mathbb{F}}
\newcommand{\artifactlink}{\href{https://github.com/LeGenAI/llm_guided_EAQECC_search/tree/arxiv-v1}{artifact}}
\newcommand{\vct}[1]{\boldsymbol{#1}} 
\newcommand{\wt}{\operatorname{wt_s}}
\newcommand{\cmins}{c^{\mathrm{s}}_{\min}}
\newcommand{\claim}[1]{\textbf{\texttt{#1}}}   

\title{Researcher-in-the-loop\\
  Evolutionary Discovery of\\
  entanglement-assisted quantum code families}
\author{Jae-Hyun Baek$^{1,2}$ \quad Byung-Sun Won$^{3}$ \quad Jon-Lark Kim$^{4,2}$\thanks{Corresponding author: \texttt{jlkim@sogang.ac.kr}} \\
$^{1}$Department of Artificial Intelligence, Sogang University, Seoul, Korea \\
$^{2}$Institute for Mathematical and Data Sciences, Sogang University, Seoul, Korea \\
$^{3}$Ewha Institute of Mathematical Sciences, Ewha Womans University, Seoul, Korea \\
$^{4}$Department of Mathematics, Sogang University, Seoul, Korea}

\begin{document}
\maketitle
\lhead{}\renewcommand{\headrulewidth}{0pt}

\begin{abstract}
LLM-guided program evolution can produce mathematical constructions,
but validity, generality, and reproducibility require distinct evidence.
We develop a researcher-in-the-loop discovery pipeline that combines
program evolution, exact evaluation, mathematical generalization,
certified nonexistence checks, and formal proof checking. We apply it
to entanglement-assisted quantum error-correcting codes (EAQECCs).
Starting from tasks that did not specify a family formula, evolved
programs produced finite qubit constructions whose structure researchers
generalized to the family $[[n,1,n-1;n-q-1]]_q$ for prime powers
$q\ge3$ at $n\ge2q+1$ and for $q=2$ at odd $n\ge5$. Lean checks the
algebraic certificate underlying this family. Together with entanglement
lifting, the family closes 54 gaps in archived EAQECC tables. The
released evidence includes 115 verified code records, ten nonexistence
results with checked DRAT proofs, family implementations, and preserved
search programs for exact checks and seeded replay. In a controlled
cohort of eight paired generation blocks, validation-selected iterative
programs achieve success in 46.1\% and 31.6\% of held-out executions at
unseen lengths $n=13,15$, compared with 19.5\% and 3.1\% for independently
generated programs. In this setting, program evolution and researcher
analysis turn finite computational discoveries into a proved family
with independently checkable evidence.
\end{abstract}

\section{Introduction}

LLM-guided program search has produced mathematical constructions and
algorithms~\citep{FunSearch,AlphaEvolve}. Their outputs can also support
general mathematical results through researcher analysis. This raises
three questions: \emph{whether evolved programs produce valid codes,
whether their structure can be generalized and proved to work beyond
the searched cases, and whether the resulting constructions and preserved
searches can be independently reproduced}. Addressing these questions
requires records of the initial task, computational discoveries and
subsequent mathematical analysis.

We study this process for entanglement-assisted quantum error-correcting
codes (EAQECCs). Early AlphaEvolve searches produced finite qubit codes
whose structure led us to an infinite qubit family. Further mathematical
generalization gives the prime-power construction and proof in
Theorem~\ref{thm:unified}. The initial instructions suggested directions
such as cyclic shifts, but did not provide this family formula.
Later instructions explicitly incorporated the discovered qubit family
(Appendix~\ref{app:prompts}), allowing it to guide further searches.

Our reproducibility contribution concerns these mathematical
constructions and their computational evidence. The \artifactlink{} supplies
code matrices, family implementations, proofs, task specifications and
preserved search programs. Readers can recompute code parameters,
construct and verify family instances, inspect the proof and replay
preserved searches. Figure~\ref{fig:overview} summarizes how researcher
analysis and program evolution interact. Section~\ref{sec:repro}
specifies which operations are reproducible. A final Lean stage checks
the algebraic construction underlying Theorem~\ref{thm:unified} (Appendix~\ref{app:lean}); model generation
remains stochastic.

Several improvements concern constructions published in 2020 and 2022
that remained listed in the EAQECC tables accessed on 17 July
2026~\citep{GHW2020,Luo2022}. Our constructions and certified exclusions
close nine previously tabulated qubit distance gaps (Table~\ref{tab:results}).
The prior constructions predate our table access by approximately four
to six years. Their dates document available constructions rather than
a complete history of improvement attempts.

Our contributions are:
\begin{itemize}\setlength{\itemsep}{1pt}
\item \textbf{A researcher-in-the-loop discovery pipeline.}
  We connect LLM-guided program evolution with researcher proposals,
  mathematical generalization and formal proof checking for EAQECC
  discovery. Preserved programs, code matrices and formal certificates
  support reconstruction and independent verification of the resulting
  constructions (Sections~\ref{sec:pipeline} and~\ref{sec:repro}).
\item \textbf{General constructions from evolved program outputs.}
  Finite qubit constructions led to the general family of
  Theorem~\ref{thm:unified}, whose algebraic construction is checked in Lean.
  The family and entanglement lifting close 54 listed gaps, while checked
  nonexistence certificates establish selected optimality and
  minimum-entanglement results (Section~\ref{sec:results}).
\item \textbf{Program development and generalization.} We compare
  independent and iterative generation under shared research guidance.
  Eight new generation blocks yield longer-code success of 46.1\% and
  31.6\% for iterative programs, versus 19.5\% and 3.1\% for independent
  proposals at $n=13,15$. Preserved lineages identify useful objective and
  search-policy changes (Section~\ref{sec:ablation}).
\end{itemize}

\section{Background and related work}\label{sec:related}

\subsection{Entanglement-assisted quantum error-correcting codes}
\label{sec:ea-background}

Quantum error correction protects quantum information sent through noisy
communication links. Early stabilizer constructions imposed algebraic
compatibility constraints that limited which classical codes could be
reused. \citet{BrunDevetakHsieh2006} showed that entanglement shared in
advance between sender and receiver relaxes these constraints, enabling
a broader class of quantum codes. This is the motivation for
entanglement-assisted quantum error-correcting codes (EAQECCs).

A code with parameters $[[n,k,d;c]]_q$ encodes $k$ logical qudits into
$n$ transmitted $q$-level systems using $c$ shared maximally entangled
qudit pairs. Each pair contains $\log_2 q$ ebits, equal to one for
qubits ($q=2$)~\citep{GHW2022}. The receiver's
halves are assumed noiseless and are not counted in $n$.

The distance $d$ determines the code's guaranteed error protection.
With ideal decoding, it corrects arbitrary errors on up to
$\lfloor(d-1)/2\rfloor$ transmitted qudits, or up to $d-1$
\emph{erasures}, where the lost qudits' locations are
known~\citep{BrunDevetakHsieh2006,GHW2022}. Increasing $d$ at fixed
$(q,n,k,c)$ strengthens protection without increasing the transmission or
entanglement budget. For example, our $[[12,1,11;8]]_2$ code raises
the listed distance from 9 to 11, increasing the guaranteed correction
of arbitrary errors from four to five transmitted qubits with the same
resources (Table~\ref{tab:results}). A verified code thus provides an
encoding specification for a quantum link with given resource and
error-protection requirements. Reducing $c$ at fixed $(n,k,d)$ instead
lowers the entanglement needed before transmission~\citep{Luo2022}.

\subsection{Code discovery as program search}
\label{sec:tables-search-background}

We search linear stabilizer EAQECCs represented by generator matrices
over $\F_q$. The LLM revises a search program that proposes matrices;
an exact evaluator computes their $(k,c,d)$. At fixed $(q,n,k,c)$,
reference tables give a constructed lower distance bound $d_\ell$ and
an upper bound $d_u$~\citep{Codetables,Luo2022}. Finding $d>d_\ell$
improves the lower bound, while reaching $d_u$ closes the gap.
We also construct codes at unlisted cells and minimize entanglement,
denoted $\cmins(n,k,d)$ within our stabilizer class. Applicable bounds
filter impossible targets~\citep{GHW2022}; Appendix~\ref{app:domain}
gives the representation, notation and comparison rules.

\textsc{FunSearch} evolves programs that construct mathematical objects or
make algorithmic decisions, using an evaluator to select useful revisions.
It produced larger cap-set constructions and improved online bin-packing
heuristics~\citep{FunSearch}; its executable outputs let researchers inspect
and develop the underlying mathematical ideas. \textsc{AlphaEvolve}
extends this approach to matrix-multiplication algorithms, data-center
scheduling, hardware-circuit simplification and model-training
kernels~\citep{AlphaEvolve,Georgiev2025}. Related frameworks include
\textsc{OpenEvolve}, \textsc{CodeEvolve}, \textsc{ShinkaEvolve} and
\textsc{TurboEvolve}, which implement or extend evaluator-guided program
evolution with population management and proposal-selection
strategies~\citep{OpenEvolve,CodeEvolve,ShinkaEvolve,TurboEvolve}.
We use \textsc{AlphaEvolve} as the evolution engine for EAQECC discovery,
supplying target parameters, an exact evaluator and researcher-directed
construction ideas.

Quantum-code discovery also uses genetic algorithms to evolve stabilizer
representations~\citep{WebsterBrowne2024}, and reinforcement learning to
optimize codes and encoders~\citep{Nautrup2019,Su2023,Mauron2023,Olle2024}.
More directly, \citet{CruzBenito2026} evolve Python generators for
bivariate bicycle codes, verify candidates, and derive structural results.
Our contribution concerns EAQECC distance--entanglement targets: exact
finite-field evaluation supports finite constructions, a prime-power
family with a Lean-checked algebraic certificate, and certified
nonexistence results. The controlled study compares independent proposals
with iterative program feedback under a shared researcher proposal.


\section{The researcher-in-the-loop discovery pipeline}\label{sec:pipeline}

Our discovery pipeline combines human-in-the-loop (HITL) research with
LLM-guided program evolution and exact verification. Researchers propose
construction ideas, inspect verified instances and revise the task between
campaigns. Figure~\ref{fig:overview} connects these interventions to the
automated search and evaluation loop. Constructed codes establish lower
distance bounds; SAT refutations with checked certificates lower upper
bounds. Independent verification and archived evidence support both routes.
Researchers generalize the discovered constructions and develop proofs,
then pass the formalized statements and proofs to the Lean checking stage.

\begin{figure}[H]
\centering
\resizebox{\linewidth}{!}{
\definecolor{ovink}{HTML}{203849}%
\definecolor{ovblue}{HTML}{2C5F8A}%
\definecolor{ovteal}{HTML}{1F7778}%
\definecolor{ovhuman}{HTML}{985B23}%
\definecolor{ovred}{HTML}{A54339}%
\begin{tikzpicture}[x=1cm,y=1cm,>=Latex,
  font=\rmfamily\fontsize{9.2}{10.8}\selectfont,text=ovink,
  flow/.style={->,draw=ovink!85,line width=.95pt},
  humanflow/.style={->,draw=ovhuman,line width=.95pt,dashed},
  small/.style={font=\rmfamily\fontsize{8.2}{9.8}\selectfont},
  label/.style={font=\rmfamily\bfseries\fontsize{9.2}{10.5}\selectfont},
  card/.style={rounded corners=3pt,draw=ovink!40,line width=.75pt,fill=white}]
\path[use as bounding box] (-.03,.24) rectangle (14.03,7.16);

\begin{scope}[yshift=.65cm]
\filldraw[rounded corners=4pt,fill=ovhuman!6,draw=ovhuman!40,line width=.65pt]
  (0,5.18) rectangle (9.95,6.48);
\fill[ovhuman!12] (.17,5.34) rectangle (2.75,6.32);
\draw[ovhuman,line width=1pt] (.53,6.00) circle (.10);
\draw[ovhuman,line width=1pt,rounded corners=2pt] (.34,5.62) -- (.34,5.75)
  .. controls (.36,5.92) and (.70,5.92) .. (.72,5.75) -- (.72,5.62);
\node[label,anchor=west,text=ovhuman] at (.88,6.03) {Researcher};
\node[small,anchor=west,text=ovhuman] at (.88,5.63) {in the loop};
\node[label,align=center,text width=3.35cm,inner sep=0pt] at (4.30,6.05)
  {Propose and revise ideas};
\node[small,align=center] at (4.30,5.48) {between campaigns};
\node[label,align=center,text width=3.35cm,inner sep=0pt] at (8.00,6.05)
  {Generalize and prove};
\node[small,align=center] at (8.00,5.48) {from discovered codes};
\draw[card,draw=ovteal!65] (10.65,5.18) rectangle (14,6.48);
\node[inner sep=0pt] at (12.325,6.00)
  {\includegraphics[width=1.55cm]{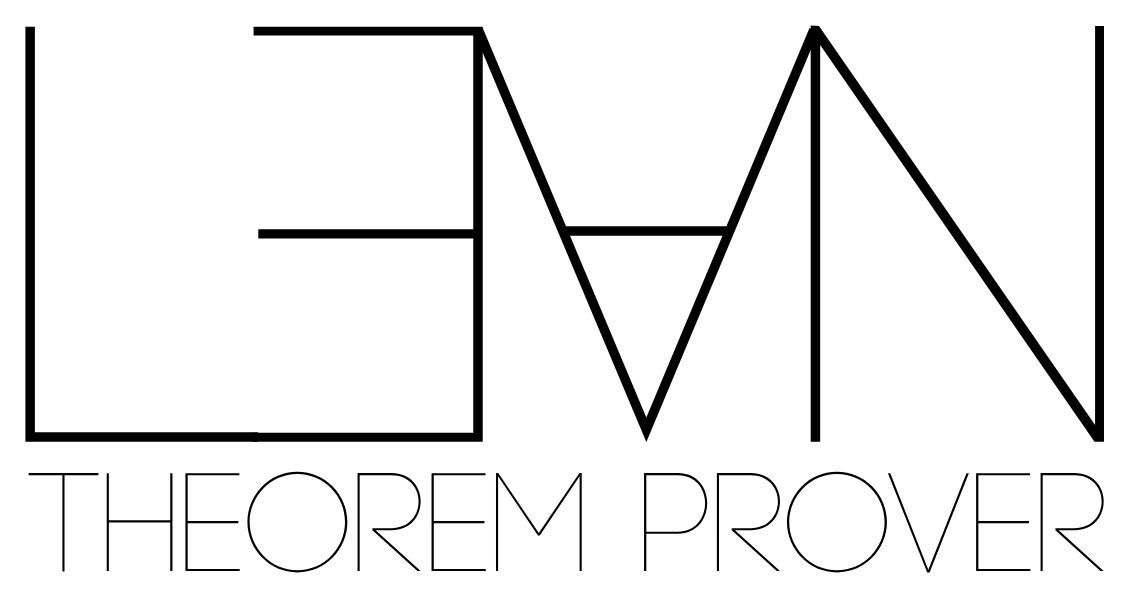}};
\node[small,text=ovteal] at (12.325,5.39) {Machine-check proofs};
\draw[flow,ovteal] (9.98,5.83)--(10.62,5.83);
\end{scope}

\foreach \x/\w in {0/2.42,2.97/3.30,6.80/3.08,10.43/3.57}
  \draw[card] (\x,.72) rectangle ++(\w,3.78);
\fill[ovink!6] (.015,3.91) rectangle (2.405,4.48);
\fill[ovblue!9] (2.985,3.91) rectangle (6.255,4.48);
\fill[ovteal!9] (6.815,3.91) rectangle (9.865,4.48);
\fill[ovink!6] (10.445,3.91) rectangle (13.985,4.48);
\node[label,anchor=center] at (1.21,4.22) {01\quad TARGET};
\node[label,anchor=center,text=ovblue] at (4.62,4.22) {02\quad EVOLVE};
\node[label,anchor=center,text=ovteal] at (8.34,4.22) {03\quad VERIFY};
\node[label,anchor=center] at (12.215,4.22) {04\quad DISCOVER};

\node[small] at (.62,3.50) {$n$};
\node[small] at (1.59,3.50) {$d_\ell\;\text{--}\;d_u$};
\draw[ovink!30,line width=.45pt] (.26,3.28)--(2.16,3.28);
\foreach \n/\lo/\hi/\y in {7/5/6/3.05,9/7/8/2.65,11/9/10/2.25}{
  \node[small] at (.62,\y) {$\n$};
  \node[small] at (1.59,\y) {$\lo\;\text{--}\;\hi$};
}
\node[small,align=center,text width=2.04cm,inner sep=0pt] at (1.21,1.40)
  {Reference tables\\Known bounds};

\draw[rounded corners=2pt,fill=ovblue!5,draw=ovblue!60,line width=.7pt]
  (3.17,3.09) rectangle (5.37,3.82);
\node[small,text=ovblue,align=center,text width=2.10cm,inner sep=0pt]
  at (4.27,3.455) {LLM evolves\\a program};
\foreach \row in {0,1,2}
  \foreach \col in {0,...,7}{
    \pgfmathtruncatemacro{\active}{mod(\col-\row+8,8)}
    \ifnum\active<3
      \fill[ovblue!80] (3.63+.16*\col,2.64-.13*\row) rectangle ++(.12,-.09);
    \else
      \fill[ovblue!10] (3.63+.16*\col,2.64-.13*\row) rectangle ++(.12,-.09);
    \fi
  }
\node[small,text=ovblue] at (4.27,2.05) {Generators};
\draw[rounded corners=2pt,fill=ovteal!6,draw=ovteal!60,line width=.7pt]
  (3.17,.96) rectangle (5.37,1.57);
\node[small,text=ovteal] at (4.27,1.265) {Exact evaluator};
\draw[->,>={Latex[length=1.2mm,width=1mm]},line width=.8pt,ovblue]
  (4.27,3.02)--(4.27,2.73);
\draw[->,>={Latex[length=1mm,width=.9mm]},line width=.75pt,ovblue]
  (4.27,1.82)--(4.27,1.64);
\draw[flow,ovblue] (5.39,1.265)--(6.07,1.265)--(6.07,3.455)--(5.40,3.455);
\node[small,rotate=90,inner sep=0pt,text=ovblue] at (5.82,2.37) {feedback};

\fill[ovteal] (7.030,3.590) rectangle ++(.084,-.084);
\fill[ovteal] (7.142,3.590) rectangle ++(.084,-.084);
\fill[ovteal] (7.254,3.590) rectangle ++(.084,-.084);
\fill[ovteal] (7.366,3.590) rectangle ++(.084,-.084);
\fill[ovteal] (7.478,3.590) rectangle ++(.084,-.084);
\fill[ovteal] (7.590,3.590) rectangle ++(.084,-.084);
\fill[ovteal] (7.702,3.590) rectangle ++(.084,-.084);
\fill[ovteal] (7.814,3.590) rectangle ++(.084,-.084);
\fill[ovteal] (7.926,3.590) rectangle ++(.084,-.084);
\fill[ovteal] (8.038,3.590) rectangle ++(.084,-.084);
\fill[ovteal] (8.150,3.590) rectangle ++(.084,-.084);
\fill[ovteal] (7.030,3.478) rectangle ++(.084,-.084);
\fill[ovteal] (7.142,3.478) rectangle ++(.084,-.084);
\fill[ovteal!10] (7.254,3.478) rectangle ++(.084,-.084);
\fill[ovteal] (7.366,3.478) rectangle ++(.084,-.084);
\fill[ovteal] (7.478,3.478) rectangle ++(.084,-.084);
\fill[ovteal] (7.590,3.478) rectangle ++(.084,-.084);
\fill[ovteal] (7.702,3.478) rectangle ++(.084,-.084);
\fill[ovteal!10] (7.814,3.478) rectangle ++(.084,-.084);
\fill[ovteal] (7.926,3.478) rectangle ++(.084,-.084);
\fill[ovteal] (8.038,3.478) rectangle ++(.084,-.084);
\fill[ovteal] (8.150,3.478) rectangle ++(.084,-.084);
\fill[ovteal] (7.030,3.366) rectangle ++(.084,-.084);
\fill[ovteal] (7.142,3.366) rectangle ++(.084,-.084);
\fill[ovteal!10] (7.254,3.366) rectangle ++(.084,-.084);
\fill[ovteal] (7.366,3.366) rectangle ++(.084,-.084);
\fill[ovteal] (7.478,3.366) rectangle ++(.084,-.084);
\fill[ovteal!10] (7.590,3.366) rectangle ++(.084,-.084);
\fill[ovteal] (7.702,3.366) rectangle ++(.084,-.084);
\fill[ovteal] (7.814,3.366) rectangle ++(.084,-.084);
\fill[ovteal] (7.926,3.366) rectangle ++(.084,-.084);
\fill[ovteal] (8.038,3.366) rectangle ++(.084,-.084);
\fill[ovteal] (8.150,3.366) rectangle ++(.084,-.084);
\fill[ovteal!10] (7.030,3.254) rectangle ++(.084,-.084);
\fill[ovteal] (7.142,3.254) rectangle ++(.084,-.084);
\fill[ovteal] (7.254,3.254) rectangle ++(.084,-.084);
\fill[ovteal!10] (7.366,3.254) rectangle ++(.084,-.084);
\fill[ovteal] (7.478,3.254) rectangle ++(.084,-.084);
\fill[ovteal] (7.590,3.254) rectangle ++(.084,-.084);
\fill[ovteal] (7.702,3.254) rectangle ++(.084,-.084);
\fill[ovteal] (7.814,3.254) rectangle ++(.084,-.084);
\fill[ovteal] (7.926,3.254) rectangle ++(.084,-.084);
\fill[ovteal] (8.038,3.254) rectangle ++(.084,-.084);
\fill[ovteal] (8.150,3.254) rectangle ++(.084,-.084);
\fill[ovteal] (7.030,3.142) rectangle ++(.084,-.084);
\fill[ovteal] (7.142,3.142) rectangle ++(.084,-.084);
\fill[ovteal] (7.254,3.142) rectangle ++(.084,-.084);
\fill[ovteal] (7.366,3.142) rectangle ++(.084,-.084);
\fill[ovteal] (7.478,3.142) rectangle ++(.084,-.084);
\fill[ovteal] (7.590,3.142) rectangle ++(.084,-.084);
\fill[ovteal] (7.702,3.142) rectangle ++(.084,-.084);
\fill[ovteal] (7.814,3.142) rectangle ++(.084,-.084);
\fill[ovteal] (7.926,3.142) rectangle ++(.084,-.084);
\fill[ovteal!10] (8.038,3.142) rectangle ++(.084,-.084);
\fill[ovteal] (8.150,3.142) rectangle ++(.084,-.084);
\fill[ovteal] (7.030,3.030) rectangle ++(.084,-.084);
\fill[ovteal] (7.142,3.030) rectangle ++(.084,-.084);
\fill[ovteal!10] (7.254,3.030) rectangle ++(.084,-.084);
\fill[ovteal] (7.366,3.030) rectangle ++(.084,-.084);
\fill[ovteal] (7.478,3.030) rectangle ++(.084,-.084);
\fill[ovteal] (7.590,3.030) rectangle ++(.084,-.084);
\fill[ovteal!10] (7.702,3.030) rectangle ++(.084,-.084);
\fill[ovteal] (7.814,3.030) rectangle ++(.084,-.084);
\fill[ovteal!10] (7.926,3.030) rectangle ++(.084,-.084);
\fill[ovteal!10] (8.038,3.030) rectangle ++(.084,-.084);
\fill[ovteal] (8.150,3.030) rectangle ++(.084,-.084);
\node[inner sep=0pt] at (9.03,3.62)
  {\includegraphics[width=1.38cm]{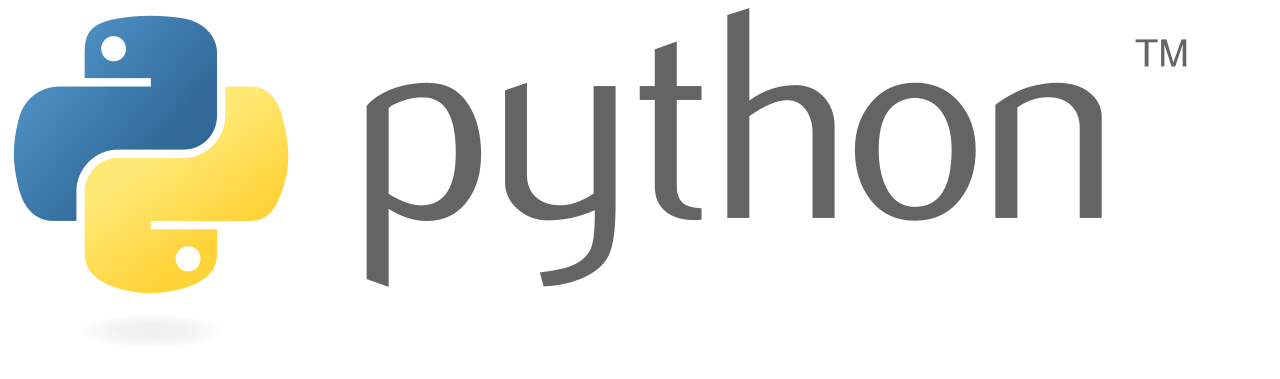}};
\node[inner sep=0pt] at (9.03,3.17)
  {\includegraphics[width=.966cm]{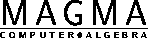}};
\node[small,text=ovteal] at (8.34,2.73) {Recompute $k,c,d$};
\draw[ovink!30,line width=.45pt] (7.00,2.50)--(9.68,2.50);
\draw[ovred!75,line width=.75pt] (7.09,1.60)--(7.09,2.27)--(7.57,2.27)
  --(7.71,2.13)--(7.71,1.60)--cycle;
\draw[ovred!75,line width=.65pt] (7.23,2.04)--(7.53,2.04) (7.23,1.88)--(7.53,1.88);
\node[small,align=left,anchor=west,text width=1.73cm,inner sep=0pt]
  at (7.94,1.99) {SAT + DRAT\\proof check};
\node[small,text=ovred,align=center,text width=2.65cm,inner sep=0pt]
  at (8.34,1.16) {Certify nonexistence};

\node[small,text=ovteal] at (11.13,3.56) {$d_\ell\uparrow$};
\node[small,text=ovred] at (13.25,3.56) {$d_u\downarrow$};
\draw[->,line width=1.1pt,ovteal] (10.92,3.26)--(11.98,3.26);
\draw[->,line width=1.1pt,ovred] (13.49,3.26)--(12.42,3.26);
\draw[ovink,line width=.8pt] (12.20,3.26) circle (.16);
\draw[ovink,line width=.7pt] (12.12,3.26)--(12.18,3.19)--(12.30,3.35);
\node[small] at (12.215,2.87) {Tighten code bounds};
\node[font=\fontsize{10.6}{12}\selectfont] at (12.215,2.38) {$[[11,1,10;8]]_2$};
\node[font=\rmfamily\bfseries\fontsize{15}{17}\selectfont,text=ovteal] at (11.37,1.77) {115};
\node[font=\rmfamily\bfseries\fontsize{15}{17}\selectfont,text=ovred] at (13.02,1.77) {10};
\node[small] at (11.37,1.24) {codes};
\node[small] at (13.02,1.24) {exclusions};

\draw[flow] (2.44,3.37)--(2.94,3.37);
\draw[flow,ovteal] (6.29,3.37)--(6.77,3.37);
\draw[flow,ovteal] (9.90,3.37)--(10.40,3.37);
\draw[flow,ovred] (1.21,.70)--(1.21,.27)--(6.52,.27)--(6.52,1.99)--(6.77,1.99);
\node[small,inner sep=1pt,text=ovred] at (3.80,.51) {nonexistence targets};

\draw[humanflow] (4.53,5.81)--(4.53,4.54);
\node[small,fill=white,inner sep=2pt,text=ovhuman] at (4.53,5.18) {task specification};
\draw[humanflow] (12.20,4.52)--(12.20,5.20)--(8.03,5.20)--(8.03,5.79);
\node[small,fill=white,inner sep=2pt,text=ovhuman] at (10.42,5.20) {inspect codes};
\end{tikzpicture}%

}
\caption{The researcher-in-the-loop discovery pipeline.}
\label{fig:overview}
\end{figure}
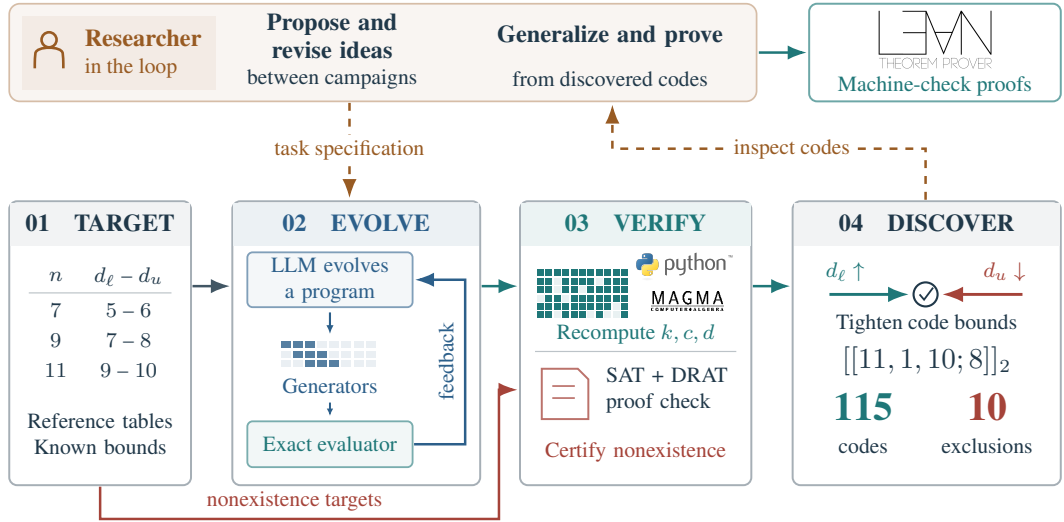


\subsection{Target selection}

We compare our constructions with the bounds listed in the EAQECC tables
(\texttt{codetables.de}), accessed on 17 July 2026. These reference bounds
define listed-gap and absent-cell targets, filtered using the EA-Plotkin
and applicable Singleton bounds. EA-Plotkin also tightens 398 upper bounds
in the reference qubit table. Appendix~\ref{app:tables} gives the table
records and filters. The artifact preserves these comparison data and
uses them by default. A second copy, retrieved on 10 September 2026,
contains identical intervals at all 11,640 qubit and qutrit entries.

\subsection{LLM-guided program evolution}

The qubit discovery campaigns use AlphaEvolve~\citep{AlphaEvolve} through
Google Cloud Vertex AI. We provide a seed program, task specification
and exact evaluator; AlphaEvolve generates revisions and archives their
source and outcomes. An LLM modifies a designated block of the seed.
Within each candidate program, an exact parameter oracle evaluates
intermediate matrices; the program uses these results to choose targets,
accept local moves and restart search. Algorithm~\ref{alg:search}, for
example, minimizes a cost combining signature mismatch and the number
of logical operators below the target distance.

After the program returns up to 16 matrices, the external evaluator
recomputes their parameters and assigns the best proposal score to the
program. Its reward favors a valid signature, fewer low-weight logical
operators and newly achieved targets. This reward guides the next program
revision; the inner cost guides matrix search within one execution.
Independent code checking follows in Section~\ref{sec:refute}.

Each candidate runs in its own subprocess with a fixed 240\,s budget
and a specified seed for its random-number generator. Failures receive a low score, and the campaign continues.
The same interface supports fixed-program replay (Section~\ref{sec:repro}).
The four campaign logs record $401$ completed candidate searches,
corresponding to $26.7$\,h of allocated search budget
(Appendix~\ref{app:campaign}).

The initial task specification did not contain the family proved in
Theorem~\ref{thm:unified}. We started with an annealing seed and instructions suggesting cyclic
shifts among other construction directions. Evolved programs produced finite qubit codes; analysis of
their structure led researchers to a general qubit family and its proof.
The later task specifications included that family and requested further
constructions, while the best program was reused as a new seed.
These records distinguish the initial discovery stage from subsequent
search using the mathematical result (Appendix~\ref{app:prompts}).
The current prime-power theorem extends the qubit construction through
further researcher-led generalization and proof.

\subsection{Independent verification and nonexistence checking}\label{sec:refute}

The Magma export supplies only a basis of the symplectic dual
$L=S^\perp$, where $S$ is the stabilizer space. Magma independently
recomputes the dual, radical and $(k,c,d)$ (Appendix~\ref{app:symplectic}).
All 148 objects---115 core codes at $q=2,3,4,5$ and 33 family instances---were re-verified with zero mismatches under Magma V2.28-20.
Separate $\F_2$ linear-algebra checks verify the two solver-refined
qubit codes; all code records can also be checked in Python using the
\artifactlink{}. Lean checks the algebraic certificate underlying
Theorem~\ref{thm:unified}. A separate Comparator run checks the submitted
proof against the exact formal goal and permitted axioms
(Appendix~\ref{app:lean}).

To test existence of an $[[n,k,d;c]]_q$ code, we encode the smaller
subspace $L=S^\perp$ in symplectic normal form. Constraints enforce
the target signature directly, so every satisfying assignment specifies
an EAQECC. Distance constraints dominate the encoding.

For an unsatisfiable encoding, the SAT solver produces a DRAT proof,
a certificate of unsatisfiability checked independently by
DRAT-trim~\citep{DRATtrim}. Together with the encoding's completeness,
this establishes the corresponding code nonexistence result. The qutrit encoding uses a coordinate-normalization
lemma to preserve completeness. Appendix~\ref{app:sat} gives the proof
scope and encoding, and Section~\ref{sec:optimality-results} reports the
certified results.

\subsection{Reproducibility}\label{sec:repro}

Readers can reconstruct two outputs: code constructions and preserved
program executions. Supplied matrices support exact $(k,c,d)$ checks;
family implementations generate instances, and Theorem~\ref{thm:unified}
proves the construction for all admissible parameters. Task specifications
and discovery logs document the progression from finite qubit codes to
the researcher-derived family (Appendix~\ref{app:campaign}). Source replay
concerns retained programs, including later programs that use this family.
It replaces wall time with a virtual clock tied to evaluator calls and
preserves the recorded random-number seeds and stopping conditions.
Three programs with four seeds each reproduce all twelve stored proposal
lists, with agreement between one and four workers. The \artifactlink{}
provides these inputs, independent parameter checks and reference-bound
audits.

\FloatBarrier
\section{Results}\label{sec:results}

Algorithm~\ref{alg:search} presents a later, high-scoring AlphaEvolve
program whose seed already incorporated the discovered qubit family.
It illustrates continued search after that generalization.
Figure~\ref{fig:lineage} and the shading show changes to target selection,
reuse of stored proposals, construction choices and generator-row
mutations. The earlier discoveries that motivated the theorem are
recorded separately (Appendix~\ref{app:campaign}). Section~\ref{sec:ablation}
evaluates iterative feedback using newly generated programs, including
Algorithm~\ref{alg:ablation-search} in Appendix~\ref{app:hitl-program}.

\begin{algorithm}[H]
\caption{Search program evolved by AlphaEvolve.}
\label{alg:search}
\small
\begin{algorithmic}[1]
\Require targets $\mathcal{T}$, time allowance $B$
\Statex $\Phi(r,t)=5000|r.c-c_t|+1000|r.k-k_t|+r.\mathrm{offending}$
\State $best\gets\{\}$; \algdelta{$b_t\gets0$ for all $t\in\mathcal{T}$}
\While{time remains}
  \State $\pi_t\propto\algdelta{$n_t^{-1.5}(1+10b_t)$}$; sample $t\sim\pi$
    \algside{Use recorded progress to change target selection.}
  \State $n\gets n_t$; $s\gets n_t-k_t+c_t$; $u\sim\mathrm{Uniform}[0,1)$
  \If{\algdelta{$u<0.2$, $t\in best$ and $best[t].o>0$}}
    \State $\mathbf{G}\gets\algdelta{$\operatorname{copy}(best[t].\mathbf{G})$}$
      \algside{Resume from this target's stored best proposal.}
  \Else
    \State $\mathbf{G}\gets\algdelta{$\textsc{Construct}(t,u)$}$
      \algside{Expand construction choices.}
  \EndIf
  \If{$\mathbf{G}$ is undefined or $\operatorname{rank}\mathbf{G}\ne s$}
    \State $\mathbf{G}\gets\textsc{RandomStabilizer}(n,s)$
  \EndIf
  \State $o\gets\Phi(\textsc{Evaluate}(n,\mathbf{G},d_t),t)$; $\textsc{KeepBest}(best,t,\mathbf{G},o)$
  \State \textbf{if} $o=0$ \textbf{then continue}
  \State $\tau\gets30$; $f\gets0$
  \While{$f<1500$ \textbf{and} time remains}
    \State $\tau\gets\max(0.1,0.995\tau)$
    \State $\mathbf{G}'\gets\algdelta{$\textsc{MixedMove}(\mathbf{G})$}$
      \algside{Add cyclic rotation of a whole generator row.}
    \State \textbf{if} $\operatorname{rank}\mathbf{G}'\ne s$ \textbf{then} $f\gets f+1$; \textbf{continue}
    \State $o'\gets\Phi(\textsc{Evaluate}(n,\mathbf{G}',d_t),t)$; $\Delta\gets o'-o$
    \State Set $(\mathbf{G},o)\gets(\mathbf{G}',o')$ with probability $\min\{1,\exp(-\Delta/\tau)\}$
    \State $f\gets0$ on an accepted strict improvement, otherwise $f\gets f+1$
    \State $\textsc{KeepBest}(best,t,\mathbf{G},o)$; \textbf{if} $o=0$ \textbf{then break}
  \EndWhile
  \State \algdelta{$b_t\gets\textsc{ProgressBonus}(best,t,b_t)$}
\EndWhile
\State \Return up to 16 stored proposals with the smallest objective values
\end{algorithmic}
\vspace{2pt}
{\footnotesize\algdelta{Shading}: principal changes from the campaign seed; routine definitions are in Appendix~\ref{app:construction-program}.}
\end{algorithm}


\begin{figure}[tb]
\centering
\resizebox{\linewidth}{!}{%
\begin{tikzpicture}[x=1cm,y=1cm,>=stealth,
  font=\rmfamily\fontsize{8.2}{9.8}\selectfont,
  heading/.style={font=\rmfamily\bfseries\fontsize{9}{10.5}\selectfont},
  rowlabel/.style={anchor=west,text width=2cm,font=\rmfamily\bfseries\fontsize{8.2}{9.8}\selectfont},
  small/.style={font=\rmfamily\fontsize{7.4}{8.6}\selectfont},
  box/.style={draw=black!28,rounded corners=2pt,fill=white,minimum height=.45cm,inner sep=4pt},
  arr/.style={->,draw=black!55,line width=.65pt}]
\path[use as bounding box] (0,-4.90) rectangle (15.8,.50);
\fill[black!5,rounded corners=3pt] (2.25,-4.85) rectangle (7.40,.48);
\fill[achieve!8,rounded corners=3pt] (7.65,-4.85) rectangle (15.8,.48);
\node[heading] at (4.82,.22) {Campaign seed};
\node[heading,text=achieve] at (11.72,.22) {Evolved program (Algorithm~\ref{alg:search})};
\draw[black!16] (2.4,-.05)--(7.25,-.05);
\draw[achieve!35] (7.8,-.05)--(15.65,-.05);
\node[rowlabel] at (0,-.78) {Target\\selection};
\node at (4.82,-.58) {$\pi_t\propto n_t^{-2}$};
\node[small] at (4.82,-1.08) {fixed length-based weights};
\node at (10.72,-.58) {$\pi_t\propto n_t^{-1.5}(1+10b_t)$};
\node[box,small,text=achieve,draw=achieve!50] (score) at (14.40,-.63) {stored cost};
\node[small,text=achieve] (bonus) at (12.65,-1.12) {near-miss bonus $b_t$};
\draw[arr,draw=achieve] (score.south) |- (bonus.east);
\draw[arr,draw=achieve] (bonus.west) -| (10.8,-.81);
\draw[black!15] (2.4,-1.47)--(7.25,-1.47);
\draw[achieve!25] (7.8,-1.47)--(15.65,-1.47);
\node[rowlabel] at (0,-2.20) {Search\\initialization};
\node[box] (fresh) at (3.85,-2.05) {fresh candidate};
\node[box] (repair) at (6.25,-2.05) {repair};
\draw[arr] (fresh)--(repair);
\node[small] at (4.82,-2.63) {new initialization on each attempt};
\node[box,draw=achieve!55] (saved) at (9.15,-2.05) {saved $best[t]$};
\node[box,draw=achieve!55] (resume) at (11.95,-2.05) {resume repair};
\node[box,draw=achieve!55] (update) at (14.55,-2.05) {update $best[t]$};
\draw[arr,draw=achieve] (saved)--(resume);
\draw[arr,draw=achieve] (resume)--(update);
\draw[arr,draw=achieve] (update.south) -- (14.55,-2.70) -- (9.15,-2.70) -- (saved.south);
\node[small,fill=achieve!8,inner sep=2pt,text=achieve] at (11.85,-2.70) {conditional reuse on a later attempt};
\draw[black!15] (2.4,-3.03)--(7.25,-3.03);
\draw[achieve!25] (7.8,-3.03)--(15.65,-3.03);
\node[rowlabel] at (0,-3.94) {Generator\\mutation};
\foreach \i/\p in {0/X,1/I,2/Z,3/Y,4/I,5/X}{
 \draw[fill=white,draw=black!20] (3.47+.45*\i,-3.62) rectangle ++(.43,.38);
 \node[small] at (3.685+.45*\i,-3.43) {$\mathrm{\p}$};}
\foreach \i/\p in {0/X,1/I,2/Z,3/Y,4/Z,5/X}{
 \draw[fill=white,draw=black!20] (3.47+.45*\i,-4.29) rectangle ++(.43,.38);
 \ifnum\i=4\fill[black!22] (3.47+.45*\i,-4.29) rectangle ++(.43,.38);\fi
 \node[small] at (3.685+.45*\i,-4.10) {$\mathrm{\p}$};}
\draw[arr] (4.82,-3.67)--(4.82,-3.86);
\node[small] at (4.82,-4.61) {local Pauli update (example)};
\foreach \i/\p in {0/X,1/I,2/Z,3/Y,4/I,5/X}{
 \draw[fill=white,draw=achieve!30] (9.50+.45*\i,-3.62) rectangle ++(.43,.38);
 \ifnum\i>3\fill[achieve!25] (9.50+.45*\i,-3.62) rectangle ++(.43,.38);\fi
 \node[small] at (9.715+.45*\i,-3.43) {$\mathrm{\p}$};}
\foreach \i/\p in {0/I,1/X,2/X,3/I,4/Z,5/Y}{
 \draw[fill=white,draw=achieve!30] (9.50+.45*\i,-4.29) rectangle ++(.43,.38);
 \ifnum\i<2\fill[achieve!25] (9.50+.45*\i,-4.29) rectangle ++(.43,.38);\fi
 \node[small] at (9.715+.45*\i,-4.10) {$\mathrm{\p}$};}
\draw[arr,draw=achieve] (10.85,-3.67)--(10.85,-3.86);
\node[small,align=left,text=achieve] at (14.00,-3.80) {rotate one row\\by whole qubits};
\node[small] at (11.72,-4.61) {row rotation added to the existing moves};
\end{tikzpicture}}
\caption{Program changes in Algorithm~\ref{alg:search} (schematic).}
\label{fig:lineage}
\end{figure}
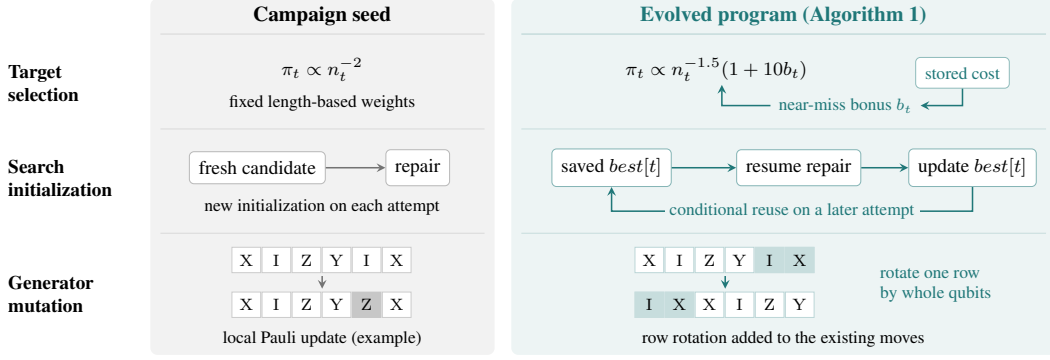


\subsection{Distance improvements}\label{sec:distance-results}

For listed cells, the tables give the July 2026 distance interval within
the previous code parameters and our verified distance in Ours $d$
(2026). Bold entries attain the applicable archived, EA-Plotkin or
certified upper bound (Section~\ref{sec:optimality-results}). References
identify previous lower-bound constructions and their dates. For
$q=4,5$, where no comparable dated entry is archived, full constructed
parameters appear in a separate \emph{New constructions} block.
No historical interval is assigned where no comparison entry exists.

\begin{table}[!ht]
\centering
\setlength{\tabcolsep}{3pt}
\renewcommand{\arraystretch}{1.18}
\begin{minipage}[t]{0.485\linewidth}
\vspace{0pt}
\centering
\footnotesize
\caption{Qubit EAQECCs ($q=2$).}
\label{tab:results}
\begin{tabular}{@{}lcc@{}}
\toprule
Previous & Ours $d$ & \multirow{2}{*}{References} \\
$[[n,k,d_\ell\text{--}d_u;c]]_q$ & (2026) & \\
\midrule
$[[7,1,5\text{--}6;4]]_2$ & $\mathbf{6}$ & \citet{GHW2020} \\
$[[9,1,7\text{--}8;6]]_2$ & $\mathbf{8}$ & \citet{Luo2022} \\
$[[10,1,7\text{--}9;6]]_2$ & $\mathbf{9}$ & \citet{Luo2022} \\
$[[11,1,9\text{--}10;8]]_2$ & $\mathbf{10}$ & \citet{Luo2022} \\
$[[12,1,9\text{--}11;8]]_2$ & $\mathbf{11}$ & \citet{Luo2022} \\
$[[12,2,7\text{--}10;9]]_2$ & $\mathbf{9}$ & \citet{Luo2022} \\
$[[13,1,9\text{--}12;8]]_2$ & $\mathbf{11}$ & \citet{Luo2022} \\
$[[13,1,11\text{--}12;10]]_2$ & $\mathbf{12}^{\dagger}$ & \citet{Luo2022} \\
$[[15,1,13\text{--}14;12]]_2$ & $\mathbf{14}^{\dagger}$ & \citet{Luo2022} \\
\bottomrule
\end{tabular}
\par\vspace{2pt}
{\scriptsize $\dagger$ Constructed by the evolved program at unseen code lengths.}
\end{minipage}\hfill
\begin{minipage}[t]{0.485\linewidth}
\vspace{0pt}
\centering
\footnotesize
\caption{Qudit EAQECCs ($q=3,4,5$).}
\label{tab:results-qudit}
\begin{tabular}{@{}lcc@{}}
\toprule
Previous & Ours $d$ & \multirow{2}{*}{References} \\
$[[n,k,d_\ell\text{--}d_u;c]]_q$ & (2026) & \\
\midrule
$[[9,1,7\text{--}8;4]]_3$ & $\mathbf{8}$ & \citet{Luo2022} \\
$[[10,1,8\text{--}9;5]]_3$ & $\mathbf{9}$ & \citet{Luo2022} \\
$[[11,1,9\text{--}10;7]]_3$ & $\mathbf{10}$ & \citet{Luo2022} \\
$[[13,1,10\text{--}12;8]]_3$ & $\mathbf{12}$ & \citet{Luo2022} \\
\midrule
\emph{New construction} & Ours $d$ & References \\
\midrule
$[[10,1,9;4]]_4$ & $9$ & Ours \\
$[[12,1,11;6]]_4$ & $11$ & Ours \\
$[[14,1,13;8]]_4$ & $13$ & Ours \\
\midrule
$[[12,1,11;5]]_5$ & $11$ & Ours \\
$[[14,1,13;7]]_5$ & $13$ & Ours \\
\bottomrule
\end{tabular}
\end{minipage}
\end{table}


The five Table~\ref{tab:results} rows with $k=1$ and $c=n-3$ are
covered by Theorem~\ref{thm:unified}; the other four retained generator
matrices were constructed by normal-form SAT. The daggered $n=13,15$
rows also have constructions from the evolved program at lengths withheld
from training and selection (Section~\ref{sec:longer-codes}). These
reconstructions are covered by the family theorem and are not counted
again among its 54 gap closures.

Our codes and the upper bounds close all nine listed qubit gaps in
Table~\ref{tab:results}. Four cells improve the previous lower bound by
two: $(n,k,c)=(10,1,6)$, $(12,1,8)$, $(12,2,9)$ and $(13,1,8)$.
For the last two cells, the EA-Plotkin bound and a checked DRAT refutation,
respectively, give final distance intervals $[9,9]$ and $[11,11]$
(Section~\ref{sec:optimality-results}). Table~\ref{tab:results-qudit}
shows four qutrit gap closures, including $[[13,1,12;8]]_3$, which raises
the listed distance from 10 to 12 at fixed length and entanglement.

The previous $[[7,1,5;4]]_2$ construction follows from the pure
$[[11,1,5]]_2$ code of \citet{CRSS1997} and the length-reduction rule of
\citet{GHW2020}; we date this EAQECC construction by the 2020 rule.
The other eight qubit baselines and the four displayed qutrit baselines
follow from the 2022 version of \citet{Luo2022}. Appendix~\ref{app:bound-history}
records the source codes and propagation rules for each cell.

The 65 core qubit code records cover 61 distinct $(n,k,c)$ cells: eight listed
and 53 absent cells. The 27 qutrit codes comprise 15 improvements
to listed entries, 11 constructions at absent cells and one code whose
distance is already listed. Multiple codes at one cell are not
counted as separate table entries. Four of these listed core cells overlap the
family results in Section~\ref{sec:family-results}. The additional
$n=15$ code record is kept in the controlled-study archive, separately from
these core-file counts.

\subsection{Code families and reduced entanglement}\label{sec:family-results}

Theorem~\ref{thm:unified} is the mathematical generalization of the
construction identified through the evolved qubit codes. It gives
$[[n,1,n-1;n-q-1]]_q$ for $q\ge3$, $n\ge2q+1$, and for $q=2$
at odd $n\ge5$. The proof in Appendix~\ref{app:theorem} establishes
the entire family; the implementations verify finite instances at
$q=2,3,4,5$. Entanglement lifting increases $c$ by one at fixed $(n,k)$
without reducing distance when $c<n-k$ (Lemma~\ref{app:lifting}).
The qubit family closes 29 listed gaps. For $q=3$, lifting from $c=n-4$
to the table's $c=n-3$ closes 25 more, for 54 in total.
This family addresses finite-length distance--entanglement tradeoffs:
at fixed $q$, its rate $k/n=1/n$ vanishes and $c/n$ tends to one.
The broader archive includes 29 binary codes with $k>1$, up to $k=6$.

Additional constructions attain $d=n-1$ at $c=n-q-2$.
The displayed $q=4,5$ examples include $[[10,1,9;4]]_4$ and
$[[12,1,11;5]]_5$, requiring four and five shared pairs, respectively.
These constructions establish achievable
entanglement costs; smaller costs at $q=4,5$ have not been excluded.
Complementary finite-field annealing searches produce 13 archived codes
at $q=4$ and 10 at $q=5$ (Appendix~\ref{app:nonbinary-search}).
Table~\ref{tab:results-qudit} selects those using the least entanglement
reached at the displayed lengths.

\begin{figure}[!ht]
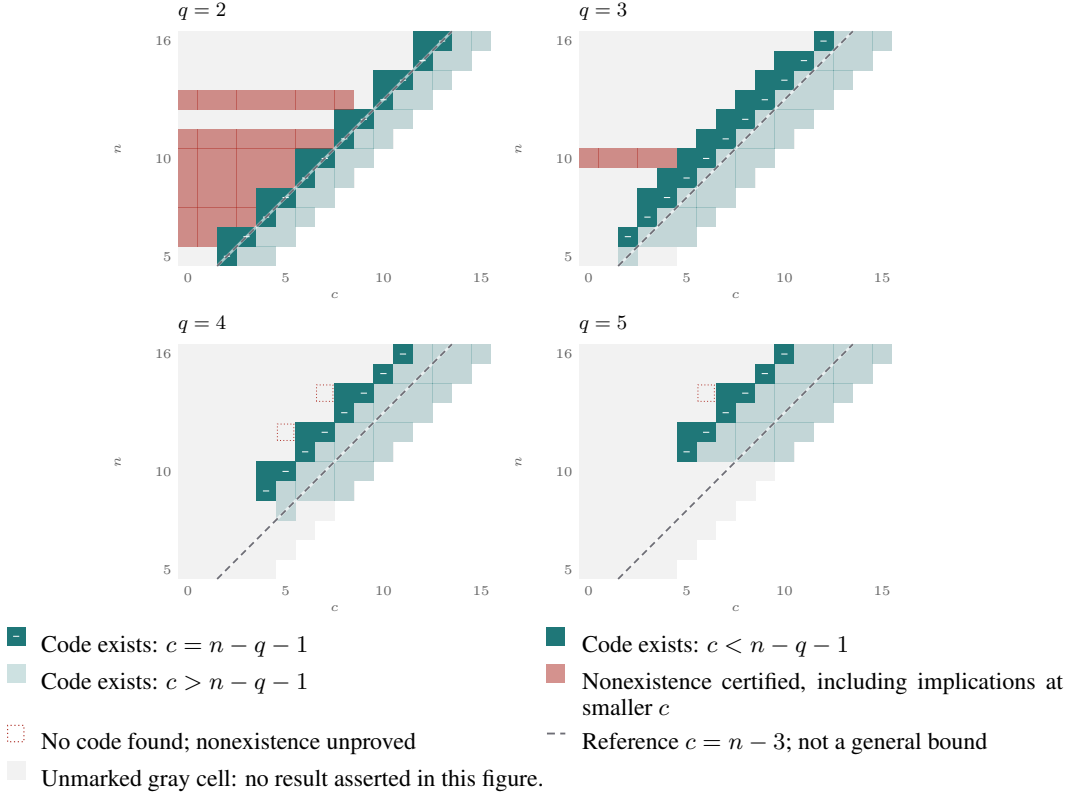

\centering
\resizebox{0.80\linewidth}{!}{%
}\par\vspace{1.5mm}
\begingroup\footnotesize
\setlength{\tabcolsep}{3pt}\renewcommand{\arraystretch}{1.18}
\begin{tabularx}{\linewidth}{@{}lX@{\hspace{1em}}lX@{}}
\tikz[baseline=-.6ex,x=2.7mm,y=2.7mm]{\fill[achieve] (0,0) rectangle (0.9,0.9); \draw[white,line width=0.35pt] (0.3,0.45)--(0.6,0.45);} & Code exists: $c=n-q-1$ & \tikz[baseline=-.6ex,x=2.7mm,y=2.7mm]{\fill[achieve] (0,0) rectangle (0.9,0.9);} & Code exists: $c<n-q-1$\\[1pt]
\tikz[baseline=-.6ex,x=2.7mm,y=2.7mm]{\fill[achieve,opacity=0.22] (0,0) rectangle (0.9,0.9);} & Code exists: $c>n-q-1$ & \tikz[baseline=-.6ex,x=2.7mm,y=2.7mm]{\fill[refute,opacity=0.55] (0,0) rectangle (0.9,0.9);} & Nonexistence certified, including implications at smaller $c$\\[1pt]
\tikz[baseline=-.6ex,x=2.7mm,y=2.7mm]{\draw[refute,line width=0.45pt,densely dotted] (0.05,0.05) rectangle (0.85,0.85);} & No code found; nonexistence unproved & \tikz[baseline=-.6ex,x=2.7mm,y=2.7mm]{\draw[bdh,line width=0.7pt,densely dashed] (0,0.45)--(0.9,0.45);} & Reference $c=n-3$; not a general bound\\[1pt]
\tikz[baseline=-.6ex]{\fill[black!5] (0,0) rectangle (2.4mm,2.4mm);} & \multicolumn{3}{l}{Unmarked gray cell: no result asserted in this figure.}
\end{tabularx}\endgroup
\caption{Linear stabilizer EAQECCs at $k=1$, target distance $n-1$.}
\label{fig:map}
\end{figure}


Figure~\ref{fig:map} concerns linear stabilizer EAQECCs with $k=1$
and target distance $n-1$. Teal at $(q,n,c)$ denotes a construction
with distance at least $n-1$. The white tick, solid fill and
pale fill distinguish $c=n-q-1$, smaller $c$ and larger $c$, respectively;
these are positions relative to the family line, not claims of minimum
entanglement. Pale cells correspond to explicit constructions or entanglement
lifting. Red cells are excluded by released, checked certificates or
by their implications at smaller $c$ through entanglement lifting. A dotted
outline records an unsuccessful search, not nonexistence. Unmarked gray
cells assert neither construction nor exclusion. The dashed $c=n-3$
line is obtained by substituting $k=1$ and $d=n-1$ into the historical
Brun--Devetak--Hsieh expression $2(d-1)\le n-k+c$ at equality.
It is a comparison line, not a general bound in this regime
(Appendix~\ref{app:tables}).

\subsection{Upper bounds and optimality}\label{sec:optimality-results}

Solver refinement settles the two remaining qubit gaps from
Table~\ref{tab:results}. Normal-form SAT finds $[[12,2,9;9]]_2$,
improving upon an intermediate $d=8$ construction from our earlier search. EA-Plotkin gives
$d\le\lfloor144/15\rfloor=9$, so the new code is distance-optimal.
At $(n,k,c)=(13,1,8)$, a DRAT-certified refutation excludes $d\ge12$;
our existing $d=11$ code is therefore optimal. Both codes are
checked by separate binary linear-algebra code, and DRAT-trim replays
the nonexistence certificate. We attribute these two closures to solver refinement.

Matching nonexistence results also establish minimum entanglement at
selected lengths. For qubits at $k=1$, $d=n-1$, the minimum is $n-3$
at $n=7,9,11$ and $n-4$ at $n=6,8,10$. This finite-length parity
pattern combines distinct constructions with certified exclusions;
Theorem~\ref{thm:unified}'s binary branch covers odd lengths.
For qutrits, the certified
exclusion of $[[10,1,9;4]]_3$ and a code at $c=5$ give
$\cmins(10,1,9)=5$. Two further exclusions, $[[5,2,4;3]]_2$ and $[[6,2,5;4]]_2$,
were also checked by exhaustive enumeration. All ten exclusions have
checked DRAT proofs (Appendix~\ref{app:sat-resources}).

Applying the EA-Plotkin bound of \citet{GuoLi2013} additionally lowers
398 listed qubit upper bounds. The auditor recomputes these corrections
against the reference table data. These corrections follow from a known
inequality and are reported as a separate source of upper bounds.

\FloatBarrier
\section{Ablation Study}\label{sec:ablation}

We test whether iterative feedback helps develop a researcher-proposed
change into an effective search program. Both policies receive the same
initial code and research guidance; their difference is how they reuse
previous programs and execution feedback.

\subsection{Program generation}
The shared proposal is to maintain and mutate a basis of
$L=S^\perp$, converting it to $S$ for exact evaluation. For
$[[9,1,8;6]]_2$ and $[[11,1,10;8]]_2$, this replaces a basis of 14 or
18 generators with four vectors. Symplectic complementation preserves
the feasible code set but changes the variables and local mutations.
A preliminary calibration compares the initial $S$-side program with
a working $L$-side control and fixes a budget at which the change can
produce measurable success. Neither the control's source nor generator
matrices are supplied to the model.

\emph{Independent proposals} always receive the initial program and its
fixed training measurements. \emph{Iterative evolution} receives the
best-so-far program and the most recent attempt, including rejected
attempts, with training feedback for both. The model can revise the
representation, objective, mutation operators and restart policy.
A proposal is a whole program revision; an evaluator call checks one
candidate matrix produced by that program. Both use Gemini 3.5 Flash
and identical tool access. There is no
controller initialization: every success must be produced by the
submitted program.

An initial four-block study was followed by a prospective eight-block
replication using new generation and execution seeds. Each policy is
allocated seven proposals per block. The replication receives 111 of
112 planned responses; one iterative pool has six proposals. All eight
blocks are retained under a rule fixed before validation and testing.
Training uses two seeds and 10,000 exact evaluations per target. Four
new validation seeds select among the three training-best distinct
programs; 32 further seeds evaluate each selection. Only training
feedback enters generation. The original and new cohorts are reported
separately (Appendix~\ref{app:ablations}).

\begin{figure}[!ht]
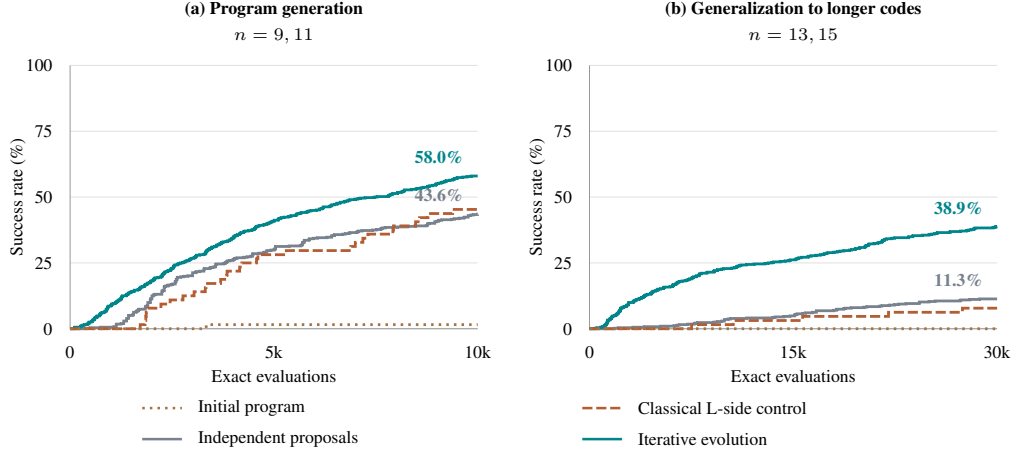

\centering
\resizebox{\linewidth}{!}{
\definecolor{rpinitial}{HTML}{A37545}
\definecolor{rpind}{HTML}{78818E}
\definecolor{rpcontrol}{HTML}{B55D32}
\definecolor{rpevo}{HTML}{00858B}


}
\caption{Prospective replication under shared research guidance.}
\label{fig:ablation-comparison}
\end{figure}


In the new cohort, success at $n=9,11$ is 72.7\% and 43.4\% for
iterative evolution, versus 60.2\% and 27.0\% for independent proposals.
The prespecified primary endpoint is the paired block difference in
mean normalized success-curve area: 0.359 versus 0.260, a gain of 0.099.
Five of eight block differences are positive; a descriptive two-sided
sign-flip test gives $p=0.328$. Generation blocks, rather than execution
seeds, are the units for policy inference. Restricting to the seven
fully received paired blocks gives a gain of 0.150.

Figure~\ref{fig:ablation-comparison} aggregates both lengths within each
panel. Each policy contributes eight selected programs and 32 runs per
length; each fixed control contributes 32 runs per length. Across the
two generation lengths, iterative selections solve 297/512 executions,
versus 223/512 for independent selections. Every saved successful code
passes an independent parameter check. Appendix~\ref{app:hitl-outcomes}
reports the earlier cohort and its fresh-seed repeat separately; the
artifact also preserves model-use accounting for both policies.

\subsection{Generalization to longer codes}\label{sec:longer-codes}
The selected programs are applied without further modification to
$[[13,1,12;10]]_2$ and $[[15,1,14;12]]_2$ at 30,000 calls.
These lengths are withheld from generation and selection; their outcomes
are secondary endpoints. In the new eight-block cohort, iterative
success is 118/256 (46.1\%) and 81/256 (31.6\%), versus 50/256
(19.5\%) and 8/256 (3.1\%) for independent proposals. Two iterative
selections achieve 32/32 at $n=13$ and at least 28/32 at $n=15$
(Appendix~\ref{app:replication}).
Their mean normalized curve areas are 0.309 and 0.193, compared with
0.099 and 0.012 for independent proposals.

A preserved lineage from the original cohort gives a source-level
diagnosis of one useful edit. Adding a penalty for distance below the
target---the only functional search change in that step---raises
$n=13$ success from 13/32 to 29/32 and $n=15$ success from 3/32 to
28/32. The validation-selected child (Algorithm~\ref{alg:ablation-search})
then bundles changes to objective weights, moves and restart behavior;
we therefore use the earlier step for edit-level attribution and report
the complete lineage in Appendix~\ref{app:hitl-program}. All resulting
codes pass independent linear-algebra checks. Table~\ref{tab:results}
includes the longer-code parameters alongside their prior intervals.

\FloatBarrier
\section{Limitations and conclusion}\label{sec:limitations}

The controlled evidence is scoped to one model, one researcher proposal
and binary targets. The eight generation blocks are the independent
units for policy inference; repeated executions measure frozen-program
variability. The policies receive equal proposal allocations but use
different token totals. Exact distance verification
enumerates symplectic duals, so its cost grows exponentially with their
dimension. The distance guarantees assume ideal decoding and noiseless
receiver-side halves; circuit costs and hardware-level error rates
are not evaluated. For these controlled studies, the artifact preserves
all generated sources, requests and selection records and supports fixed-program replay;
stochastic model sampling is documented rather than expected to return
byte-identical programs.

The central mathematical result developed from code constructions
produced by evolved programs. Researchers analyzed their structure,
derived a qubit family, and extended and proved the prime-power
construction of Theorem~\ref{thm:unified}. Later prompts reused the
qubit family, showing how a result obtained through program search can
guide further research. The family and entanglement lifting close 54
listed gaps. Normal-form SAT completes two final qubit closures, ten
checked DRAT refutations certify selected upper bounds and minimum
entanglement costs, and EA-Plotkin tightens 398 listed upper bounds.

Supplied matrices, family implementations and independent verifiers
support reconstruction and checks of finite instances; preserved programs
support seeded replay. The proof covers every admissible parameter.
Independent Magma verification reports zero mismatches across 148
objects, and Lean checks the algebraic certificate underlying the
researcher-derived theorem. The prospective eight-block cohort identifies
useful program revisions and higher held-out longer-code success among
iterative selections under shared research guidance. Applying the
workflow to another problem still requires its own representation,
evaluator and mathematical analysis.

\subsection*{AI use statement}
We used Gemini models through AlphaEvolve and a local controller to
propose and revise search programs under researcher guidance and exact
evaluation feedback (Sections~\ref{sec:pipeline} and~\ref{sec:ablation}).
LLM-based assistants, including Codex, supported implementation,
mathematical reasoning and manuscript preparation. Generated codes were
independently checked; Theorem~\ref{thm:unified}'s algebraic construction
was verified in Lean. The authors reviewed AI-assisted outputs and take
responsibility for the paper's claims, proofs, code and reported results.

\subsection*{Ethics statement}
This work studies algebraic code constructions and search programs;
it does not involve experiments on human participants or use
participant-level or personal datasets. Its potential benefit is to
improve methods for quantum error correction and communication, while
the reported idealized parameters do not establish hardware-level
reliability or operational security. Research-integrity considerations
are addressed through dated sources, explicit provenance, independent
verification, and separation of new parameter claims from retrospective
reconstruction. Unsuccessful searches are distinguished from proved
nonexistence. The controlled-study records preserve failed and excluded
attempts, unavailable responses and protocol amendments, and the released
records omit account credentials and machine-local configuration.
Researcher and AI contributions are disclosed, external sources are
cited, and the artifact retains software licensing information. Compute
allowances and model usage are documented in Appendices~\ref{app:campaign}
and~\ref{app:ablations}.

\subsection*{Reproducibility statement}
The artifact supports reconstruction and independent checking of the
mathematical outputs: generator matrices specify the finite codes,
implementations generate family instances, and Appendix~\ref{app:theorem}
provides the general construction and proof of Theorem~\ref{thm:unified}.
The initial and later task specifications in Appendix~\ref{app:prompts}
document the transition from initial search directions to reuse of
the researcher-derived qubit family. Preserved programs can be replayed with fixed
seeds and evaluation budgets as described in Section~\ref{sec:repro}
and the \artifactlink{}. Appendix~\ref{app:ablations} and the accompanying
study records supply generation requests, selections and held-out
outcomes for the separate controlled comparison. Computational checks
verify finite cases; general validity rests on the proof.
Appendix~\ref{app:lean} gives the exact formal goal and proof structure
for Theorem~\ref{thm:unified}. The artifact includes the Lean sources,
pinned dependencies and separate kernel and Comparator verification records.
Fresh LLM sampling is not part of this replay. External solver
and Magma checks are reported separately from completed Python checks.
The artifact is released under the MIT license and is available at
\url{https://github.com/LeGenAI/llm_guided_EAQECC_search/tree/arxiv-v1}.

\subsection*{Acknowledgments}
We thank the maintainers of the tables of best-known
codes~\citep{Codetables}.
J.-L.~Kim was supported in part by the BK21 FOUR (Fostering Outstanding
Universities for Research) funded by the Ministry of Education (MOE, Korea),
National Research Foundation of Korea (NRF) under Grant
No.~4120240415042, Basic Science Research Program through the National
Research Foundation of Korea (NRF) funded by the Ministry of Science and ICT
under Grant No.~RS-2025-24534992 and Global - Learning \& Academic research
institution for Master's$\cdot$PhD students, and Postdocs (LAMP) Program of
the National Research Foundation of Korea (NRF) grant funded by the Ministry
of Education (No.~RS-2024-00441954).


\clearpage
\bibliographystyle{conference}
\bibliography{refs}

\appendix
\raggedbottom
\section{The domain, the tables, and the targets}\label{app:domain}

\subsection{Codes as subspaces of $\F_q^{2n}$}\label{app:symplectic}
For prime-power $q$, our search uses $\F_q$-linear stabilizer EAQECCs.
Stabilizer generators are products of Pauli operators defining
error-detection constraints. Entanglement assistance permits noncommuting
generators on the transmitted qudits by completing them with operators
on the receiver's halves into commuting constraints. Their commutation
relations determine the entanglement cost~\citep{BrunDevetakHsieh2006,WildeBrun2008}.
Nonadditive EAQECCs form a different search space~\citep{Shin2011}.

We write vectors in bold, scalar coordinates and coefficients in italic,
and subspaces with italic capitals such as $S,L,R$.
Modulo phase, a Pauli operator is represented by
$\vct{u}=(x_1,z_1,\ldots,x_n,z_n)\in\F_q^{2n}$. For qubits, the pairs
$(0,0)$, $(1,0)$, $(0,1)$ and $(1,1)$ represent $I,X,Z,Y$, respectively. The
symplectic product with $\vct{v}=(x'_1,z'_1,\ldots,x'_n,z'_n)$ is
\[
 \langle \vct{u},\vct{v}\rangle=\sum_i(x_i z_i'-z_i x_i'),
\]
and the symplectic weight $\wt(\vct{u})$ counts positions with
$(x_i,z_i)\ne(0,0)$. For qubits, zero symplectic product means that
the Pauli operators commute.

Let $S$ be the generator span and set $s=\dim S$. Its symplectic dual
$S^\perp$ consists of vectors orthogonal to $S$.
For any subspace $U\le\F_q^{2n}$, define its radical by
$\operatorname{rad}(U):=U\cap U^\perp$. Thus
$S_{\mathrm{iso}}=\operatorname{rad}(S)$ contains stabilizer
operations with trivial logical action. Vectors in
$S^\perp\setminus S_{\mathrm{iso}}$ represent undetectable errors
with nontrivial logical action. The code parameters are
\[
  c=\frac{s-\dim S_{\mathrm{iso}}}{2},\qquad
  k = n - s + c, \qquad
  d = \min\{\wt(\vct{v}) : \vct{v} \in S^{\perp}\setminus S_{\mathrm{iso}}\}.
\]
Finite-field linear algebra gives the dual and radical. With
$r=\dim L=n+k-c$, the dual has $q^r$ vectors, of which
$q^r-q^{n-k-c}$ are non-radical; their minimum weight gives $d$.
For Theorem~\ref{thm:unified}, $r=q+2$ is independent of length at fixed $q$.
At every binary ablation length, $r=4$: 16 dual vectors and 12
non-radical vectors. The controlled evaluator computes the radical
from the Gram kernel, avoiding enumeration of $S$. The archived
discovery helper instead enumerates both $S$ and $L$, totaling
$2^s+2^r$ vectors before filtering. Thus implementation and subspace
dimensions determine cost, independently of the search representation.
Parameters are recomputed from generators, not search-program reports.

To propagate constructions and exclusions across entanglement costs,
we use the following form of the standard rule in
\citet[Section IV, rule (4)]{Luo2022}. We include its symplectic argument
for completeness.

\begin{lemma}\label{app:lifting}
An $\F_q$-linear stabilizer EAQECC with parameters $[[n,k,d;c]]_q$,
$k\ge1$ and $c<n-k$, gives an $\F_q$-linear stabilizer EAQECC with
parameters $[[n,k,d';c+1]]_q$ and $d'\ge d$.
\end{lemma}
\begin{proof}[Sketch of proof]
Set $L=S^\perp$ and $R=\operatorname{rad}(L)$, with $\dim R=j=n-k-c>0$.
Choose $L=U\oplus R$. The form on $U$ is nondegenerate, since a
vector of $U$ orthogonal to $U$ also lies in $R$, hence in $U\cap R=\{0\}$; hence $\dim U=2k$.
For a codimension-one subspace $R'\le R$, put $L'=U\oplus R'$ and
$S'=(L')^\perp$. Then $\operatorname{rad}(L')=R'$ and
\[
 L'\setminus R'\subseteq L\setminus R,
 \qquad \dim S'=s+1,\qquad \dim R'=j-1.
\]
The inclusion gives $d'\ge d$, while the parameter formulas give
$c'=((s+1)-(j-1))/2=c+1$ and $k'=n-(s+1)+c'=k$.
The ambient space is unchanged, so the length remains $n$.
\end{proof}
Repeating the operation gives distance at least $d$ for every
$c'\in\{c,\ldots,n-k\}$. Conversely, nonexistence at cost $c$ and
distance at least $d$ excludes every smaller cost at the same $(n,k)$:
any such code could be lifted to cost $c$. This justifies the upward
construction propagation and downward exclusion propagation in
Figure~\ref{fig:map}. For the qutrit family, one step reduces the radical
dimension from three to two and raises $c=n-4$ to $n-3$. The 25 open
reference cells have $n=12,\ldots,36$, all within Theorem~\ref{thm:unified}'s
range; their listed upper bound $n-1$ makes each lifted lower bound exact.

\subsection{The closed form, and why it works}\label{app:theorem}
Early evolved constructions $[[9,1,8;6]]_2$ and $[[11,1,10;8]]_2$
(Appendix~\ref{app:campaign}) motivated the odd-length qubit family
recorded in later tasks. The construction below generalizes it to prime
powers. Its algebraic properties have also been independently rechecked
with the Lean kernel; Appendix~\ref{app:lean} gives the formal goal,
proof and correspondence with this construction.

Fix a prime power $q=p^e$, a length $n\ge 2q+1$, and an enumeration
$u_1,\dots,u_q$ of $\F_q$. On qudits $1,\dots,2q$ take $q$ disjoint pairs
and let $\vct{\rho}_i$ be $Z$-type on pair $i$ with $z$-values $(1,-1)$.
Juxtaposition and powers denote concatenation of coordinate pairs;
$(0,1)^0$ is empty, so the tail at $n=2q+1$ contains only $(0,\lambda)$. Put
\[
  \vct{a}=(1,0)^{n},\qquad
  \vct{b}=(u_1,1)(u_1,z_2)\,\big|\cdots\big|\,(u_q,1)(u_q,z_2)\,\big|\,
    (0,1)^{n-2q-1}(0,\lambda),
\]
with $z_2=1$ for odd $p$ and $z_2=0$ for $p=2$, and
$L=\operatorname{span}\{\vct{a},\vct{b},\vct{\rho}_1,\dots,\vct{\rho}_q\}$.
\begin{theorem}\label{thm:unified}
Let $q$ be a prime power. For $q\ge3$ with $n\ge2q+1$, or $q=2$
with odd $n\ge5$, one can choose $\lambda\in\F_q^{\times}$ so that
$L$ is the normalizer side of an $[[n,1,n-1;n-q-1]]_q$ EAQECC.
For $q\ge3$, its entanglement cost $c$ is $q-2$ shared pairs below
the threshold in the BDH expression; for $q=2$, it meets that threshold.
\end{theorem}
\begin{proof}
The two vectors $\vct{a},\vct{b}\in\F_q^{2n}$ and the $q$ radical
candidates $\vct{\rho}_1,\ldots,\vct{\rho}_q$ are defined above.
Write $R_0=\operatorname{span}\{\vct{\rho}_1,\ldots,\vct{\rho}_q\}$.
Each $\vct{\rho}_i$ has two nonzero $Z$ components on a distinct
pair of qudits, so these $q$ vectors are linearly independent.
Moreover, $\vct{a}$ and $\vct{b}$ have equal $X$ components within
each pair. Their symplectic products with $\vct{\rho}_i$ therefore
vanish. The radical candidates are mutually orthogonal, since they
are all $Z$-type.

The symplectic product of the two vectors is the sum of the $Z$
components of $\vct{b}$:
\[
 \delta:=\langle\vct{a},\vct{b}\rangle
 =q(1+z_2)+(n-2q-1)\,1_{\F_q}+\lambda
 =(n-1)\,1_{\F_q}+\lambda.
\]
Here integers are interpreted in $\F_q$, and $q\,1_{\F_q}=0$
because its characteristic divides $q$. If $q\ge3$, at most one of
the $q-1\ge2$ nonzero choices of $\lambda$ makes $\delta=0$;
choose any other nonzero value. If $q=2$ and $n$ is odd, the only
choice $\lambda=1$ gives $\delta=1$. For even binary $n$, it gives
$\delta=0$: this particular construction loses its nondegenerate pair.
Thus $\delta\ne0$ in every case covered by the theorem.

To establish independence of the full generating set, suppose
$\alpha\vct{a}+\beta\vct{b}+\sum_i t_i\vct{\rho}_i=\vct{0}$.
Since $n\ge2q+1$, the final qudit is outside the support of every
$\vct{\rho}_i$. Its two components are $(\alpha,\lambda\beta)$,
so $\alpha=\beta=0$. Independence of the radical candidates then
implies $t_i=0$ for every $i$. Consequently $\dim L=q+2$.
The restriction of the symplectic form to $L$, in this basis, has
one nondegenerate $2\times2$ block
$\left(\begin{smallmatrix}0&\delta\\-\delta&0\end{smallmatrix}\right)$
and a $q\times q$ zero block. Hence its rank is two and
$\operatorname{rad}(L)=R_0$.

Every vector outside this radical has the unique form
\[
 \vct{w}=\alpha\vct{a}+\beta\vct{b}
       +\sum_{i=1}^q t_i\vct{\rho}_i,
 \qquad (\alpha,\beta)\ne(0,0).
\]
All $n-2q$ tail qudits are nonzero: their components are
$(\alpha,\beta)$, except for the final $(\alpha,\lambda\beta)$.
On pair $i$, the two coordinate pairs are
\[
 (\alpha+\beta u_i,\,\beta+t_i),\qquad
 (\alpha+\beta u_i,\,\beta z_2-t_i).
\]
If $\beta=0$, then $\alpha\ne0$ and all $n$ qudits are nonzero.
If $\beta\ne0$, exactly one index $i$ satisfies
$u_i=-\alpha/\beta$, since $u_1,\ldots,u_q$ enumerate all $q$ elements of the finite field $\F_q$,
and $-\alpha/\beta\in\F_q$.
Every other pair has two nonzero $X$ components. On the exceptional
pair, both $Z$ components could vanish only if
$t_i=-\beta=\beta z_2$. This is impossible because
$1+z_2\ne0$: it equals $2$ in odd characteristic and $1$ in
characteristic two. At most one qudit can therefore be zero, giving
$\wt(\vct{w})\ge n-1$.

This lower bound is attained. Fix any index $i$, choose
$\beta=1$, $\alpha=-u_i$, $t_i=-1$, and $t_h=0$ for $h\ne i$.
The first qudit of pair $i$ is zero and its second has nonzero
$Z$ component $1+z_2$. All other pairs have nonzero $X$ components,
and the tail remains nonzero. This vector lies outside $R_0$ and
has weight exactly $n-1$.

Finally, set $S=L^\perp$. Nondegeneracy of the ambient symplectic
form gives $S^\perp=L$ and $\dim S=2n-q-2$. Also
$S\cap S^\perp=L^\perp\cap L=R_0$, whose dimension is $q$.
The parameter formulas in Appendix~\ref{app:symplectic} now give
\[
 c=\frac{2n-q-2-q}{2}=n-q-1,\qquad
 k=n-(2n-q-2)+c=1,\qquad d=n-1.
\]
For these parameters,
$2(d-1)-(n-k+c)=(2n-4)-(2n-q-2)=q-2$.
Thus the BDH threshold exceeds $c$ by $q-2$ shared pairs for $q\ge3$
and met with equality for $q=2$, as stated.
\end{proof}
The auditor instantiates this construction and verifies every member
exactly for $q\in\{2,3,4,5\}$ and $2q+1\le n\le 2q+8$ (claim
\claim{families}); the theorem covers every prime power.

\subsection{Reference bounds and target selection}\label{app:tables}
The reference tables record, for each parameter cell $(q,n,k,c)$,
the listed lower and upper distance bounds and their construction sources.
Our comparison uses the bounds accessed on 17 July 2026. The reference
qubit table contains 563 entries whose distance bounds differ by one. The artifact
preserves these table data and a second copy retrieved on 10 September
2026; both cover qubit codes with $n\le64$ and qutrit codes with $n\le36$.
The reproduction command uses the July reference data unless another
retrieval date is explicitly selected. Schema checks reject malformed
records before classifying improvements or previously unlisted cells.

An entry with distance interval $[d_\ell,d_u]$ closes when construction
and nonexistence bounds meet. For $(q,n,k,c)=(2,7,1,4)$, a verified
$[[7,1,6;4]]_2$ code closes the listed $5\le d\le6$ gap.
Cells absent from the reference tables are called \emph{record targets} in the
archive; absence alone does not establish novelty relative to all
published constructions.

The familiar BDH Singleton expression $2(d-1)\le n-k+c$ is not a
general bound in the large-distance regime~\citep{Grassl2021}. Our
principal family has $d=n-1$, so we use the applicable generalized
Singleton bounds of \citet{GHW2022}. The $[[12,1,11;5]]_5$ code,
for example, uses four fewer pairs than substitution into the BDH
expression would suggest; it does not contradict a valid upper bound.

For qubit ($q=2$) linear stabilizer EAQECCs with $k\ge1$, the
EA-Plotkin bound~\citep{GuoLi2013} gives
$d\le\lfloor3\cdot4^{k-1}n/(4^k-1)\rfloor$.
It lowers 398 upper bounds in the reference \emph{qubit} table.
The qubit target list retains cells not excluded by this inequality,
the applicable Singleton bounds, or our certified refutations.

\subsection{Sources and dates of previous bounds}\label{app:bound-history}
The dates in Tables~\ref{tab:results} and~\ref{tab:results-qudit} refer to
verified versions of public
sources, not necessarily to the first discovery of each EAQECC parameter
set. The reference qubit table fixes which lower and upper bounds we compare
against. For $(7,1,4)$, the underlying pure $[[11,1,5]]_2$ code is
recorded by \citet{CRSS1997}; applying Theorem~7 of \citet{GHW2020}
with $\ell=4$ establishes the EAQECC lower bound by 15 October 2020.
The elapsed time to 17 July 2026 is 5.8 years when rounded to one decimal.

The other eight qubit rows use the July 12, 2022 version of \citet{Luo2022}.
For odd $n\in\{9,11,13,15\}$, Table~I contains
$[[n-1,0,n-1;n-3]]_2$. Quantum shortening gives
$[[n-2,1,n-2;n-3]]_2$, and two length extensions yield the previous
$[[n,1,n-2;n-3]]_2$ bound. The two even-length rows arise similarly:
\[
\begin{aligned}
 [[8,0,8;6]]_2 &\longrightarrow [[7,1,7;6]]_2
                  \longrightarrow [[10,1,7;6]]_2,\\
 [[10,0,10;8]]_2 &\longrightarrow [[9,1,9;8]]_2
                  \longrightarrow [[12,1,9;8]]_2.
\end{aligned}
\]
A fourth length extension of $[[9,1,9;8]]_2$ also gives the previous
$[[13,1,9;8]]_2$ bound. For the $k=2$ row, the source's
$[[17,0,12;9]]_2$ code gives
\[
 [[17,0,12;9]]_2 \longrightarrow [[13,4,8;9]]_2
 \longrightarrow [[13,2,8;9]]_2 \longrightarrow [[12,2,7;9]]_2
\]
by four quantum shortenings, taking a subcode, and puncturing.
The rules and source codes are both present in the 2022 version,
giving 4.0 years between publication and our table access date. These derivations match
the provenance strings in the archived EAQECC table. They establish that
the baselines were available by the stated dates and still listed in
2026; they do not establish a complete chronology of intervening work.

For the displayed qutrit rows, Table~II of the same 2022 source
contains $[[10,0,8;4]]_3$, $[[10,0,9;6]]_3$, and
$[[13,1,10;8]]_3$. Quantum shortening of $[[10,0,8;4]]_3$ gives $[[9,1,7;4]]_3$.
Separately, increasing the dimension of the same $[[10,0,8;4]]_3$
source code using one additional shared pair gives $[[10,1,8;5]]_3$
\citep[Section IV, rule (6)]{Luo2022}. Applying this dimension-increase rule to
$[[10,0,9;6]]_3$ and then extending its length yields
$[[11,1,9;7]]_3$. The fourth baseline is the listed
$[[13,1,10;8]]_3$ code itself. The $q=4,5$ rows have no prior entry in
the archived tables and are not included in these historical comparisons.

\section{The discovery pipeline in detail}\label{app:pipeline-detail}

\subsection{Campaign configuration and verification}\label{app:campaign}
Across the four discovery campaigns, Google Cloud Vertex AI's
AlphaEvolve service sampled
\texttt{gemini-3.5-flash} and \texttt{gemini-3.1-pro-preview} equally,
with four generation and four evaluation workers. The four campaigns
requested $120$, $120$, $80$ and $80$ programs; worker logs record
$120$, $121$, $80$ and $80$ completed searches (401 total). Service
counters are $121$, $122$, $81$ and $81$, each including an initial
program with a supplied score. Completed searches determine the nominal
allowance: $401\times240$\,s, or $26.73$\,h. Archived records report
$1.87$\,M input and $3.74$\,M output tokens. Campaigns 1--2 share one
task specification and campaigns 3--4 another (Appendix~\ref{app:prompts}).

Campaign 1 records 120 completed searches. The first new
$[[9,1,8;6]]_2$ and $[[11,1,10;8]]_2$ codes precede completions 4
and 119, respectively: 0.267\,h and 7.933\,h of allocated
search budget at 240\,s per completed search. The logs do not provide
elapsed discovery times. A later task records $[[7,1,6;4]]_2$ as an
earlier construction without a first-hit budget. The $[[13,1,12;10]]_2$
record instead cites the closed-form theorem as its source.
Initial and later tasks document the transition from finite constructions
to the researcher-derived family. Algorithm~\ref{alg:search} is a later
retained program; the logs do not identify a surviving source for every
early discovery. Section~\ref{sec:ablation} uses separately generated,
preserved programs and its own protocol.

For the complementary finite-field searches at $q=4,5$,
simulated annealing searches\label{app:nonbinary-search}
$L=\operatorname{span}\{\vct\rho_1,\ldots,\vct\rho_j,\vct a,\vct b\}$
with $j=n-1-c$ fixed, disjoint $Z$-pair rows, each with entries $(1,1)$.
Moves resample up to three entries of $\vct a$ or $\vct b$.
The objective counts non-radical span vectors below distance $n-1$ by
exact enumeration and Gram tests. Full-rank and
$\langle\vct a,\vct b\rangle_s\ne0$ checks prune proposals; Metropolis acceptance, cooling and restarts guide search.
Before saving, a separate parameter calculation checks $(k,c,d)=(1,c,n-1)$.
All 23 archived matrices have this form and pass the independent Magma
check. The \artifactlink{} supplies the routines and code-to-record mappings.

The independent Magma check receives generators only and re-derives every quantity we
report. Symplectic weight is obtained by embedding
$(x_i,z_i)\mapsto x_i+\omega z_i$ into $\mathrm{GF}(q^2)$, choosing
$\omega\notin\F_q$. Then $\{1,\omega\}$ is an $\F_q$-basis, so the image
is zero exactly when $(x_i,z_i)=(0,0)$. Magma's own \texttt{Weight()}
therefore counts nonzero coordinate pairs from the supplied generators.

Algorithm~\ref{alg:search} summarizes the archived evolved construction
program; its preserved source is available in the
\artifactlink{}.\label{app:construction-program}
The generator matrix $\mathbf{G}$ spans $S$, and $best[t]$ stores the
lowest-cost proposal seen for target $t$. The objective remains
$5000|r.c-c_t|+1000|r.k-k_t|+r.\mathrm{offending}$, with cost $10^9$
for an evaluator error. Here $r.\mathrm{offending}$ counts
$\vct v\in S^\perp\setminus S_{\mathrm{iso}}$ with $\wt(\vct v)<d_t$;
it is an integer count. After return, the external evaluator instead
maximizes $1000-\mathrm{offending}$ for matching $(k,c)$, adding $10^6$
for a newly achieved archive target. A signature mismatch scores
$-5000|r.c-c_t|-1000$; the best proposal sets the program score.
Both calls use the exact parameter routine; Magma supplies the separate
implementation check.
The campaign seed used fixed target weights
$n_t^{-2}$. The evolved program uses normalized weights proportional to
$n_t^{-1.5}(1+10b_t)$, initially with $b_t=0$. After a repair attempt,
\textsc{ProgressBonus} uses the stored best cost $o_t^\star$:
\[
 b_t\gets
 \begin{cases}
  (1+o_t^\star)^{-1},&0<o_t^\star<100,\\
  0,&o_t^\star=0,\\
  b_t,&\text{otherwise}.
 \end{cases}
\]
If no proposal was recorded, take $o_t^\star=10^9$. A proposal with
zero cost at its initial evaluation takes the earlier \textbf{continue}
branch and skips this bonus update, exactly as in the source.

For a new attempt, the program samples one scalar $u$ and tests branches
in order. If $u<0.2$ and a stored positive-cost proposal exists, it resumes
from that matrix. Otherwise, \textsc{Construct} uses a cyclic construction
when $u<0.4$ and $s\le2n$, a divisor-stride construction when $u<0.6$
and $s\le2n$, or a block construction when $u<0.8$. If none applies,
or the resulting rank is not $s$, it uses a random stabilizer.
The construction branches and their validity checks remain subject to the
same exact evaluator as the seed.

The cyclic branches use one or two nonzero base vectors. Row $i$ takes
base $i\bmod m$ and rotates it by
$2h\lfloor i/m\rfloor\bmod 2n$ bits, where $m$ is the number of bases.
The ordinary cyclic branch sets $h=1$; the added branch samples a positive
proper divisor of $n$. The block branch extends the allowed dual dimension
from 8 to 12 and block sizes from $\{2,4,6\}$ to $\{2,4,6,8\}$.

During repair, \textsc{MixedMove} samples $v\in[0,1)$: if $v<0.2$
and $s\ge2$, XOR two distinct rows; otherwise, if $v<0.4$, rotate one
row by $1,\ldots,n-1$ qubits uniformly; otherwise, apply a weight-1--3
Pauli perturbation to one row. The rotation changes one generator,
not all coordinates of the code. Rank-deficient proposals increment $f$
without evaluation. The inherited Metropolis loop cools by $0.995$ to
$0.1$ and restarts after 1,500 attempts without an accepted strict
improvement. It retains each target's best accepted proposal and returns
at most 16. Figure~\ref{fig:lineage} shows source changes, not their
individual performance effects.

\subsection{The refutation encoding}\label{app:sat}
Encode $L=S^\perp$ using $j=n-k-c$ radical generators
$\vct\rho_1,\ldots,\vct\rho_j$ and $k$ hyperbolic pairs
$(\vct u_i,\vct v_i)$. Gram constraints enforce rank $2k$ and pair
independence. For qubits, all $2^{2k+j}-2^j$ non-radical combinations
must have weight at least $d$. Gray-code chaining changes one generator
contribution at a time; Sinz counters enforce weight. When $j\ge1$,
local symplectic maps and coordinate permutations make $\vct\rho_1$
a nonzero $Z$-prefix. For its bits $(x_i,z_i)$, the CNF sets
$x_i=0$, $z_1=1$ and adds $(z_i\lor\neg z_{i+1})$ for $i<n$.
These clauses allow every nonempty prefix in one CNF, without fixing its
weight or making the other radical rows $Z$-only.

For qutrits, each field element is one-hot encoded. The radical basis is
in reduced row echelon form, and modular state chains impose its
Gram constraints. For $k=1$, it suffices to test the four projective
logical representatives $\vct{u},\vct{v},\vct{u}+\vct{v},\vct{u}+2\vct{v}$, each shifted by every radical
vector. Their nonzero scalar multiples have the same weights. At
$d=n-1$, pairwise clauses prevent a representative from vanishing at two
coordinates. The following lemma permits an additional normalization.

\begin{lemma}[Coordinate normalization]\label{lem:coordinate-normalization}
Let $L\le\F_q^{2n}$ have radical $R$ with $\dim L=\dim R+2$ and
$\dim R>2$. If every vector in $L\setminus R$ has symplectic weight at
least $n-1$, then local symplectic transformations can make $R$ entirely
$Z$-type.
\end{lemma}
\begin{proof}[Sketch of proof]
By definition, $R=\operatorname{rad}(L)=L\cap L^\perp$.
Let $\pi_i:\F_q^{2n}\to\F_q^2$ be $\pi_i(\vct v)=(x_i,z_i)$,
and $\pi_{ih}=(\pi_i,\pi_h):\F_q^{2n}\to\F_q^4$ for $i\ne h$.
Writing $K_{ih}=\ker(\pi_{ih}|_L)$, the distance hypothesis gives
\[
 K_{ih}\subseteq\{\vct v\in L:\wt(\vct v)\le n-2\}\subseteq R,
 \qquad K_{ih}=\ker(\pi_{ih}|_R).
\]
Rank--nullity therefore yields
\[
 \dim\pi_{ih}(L)-\dim\pi_{ih}(R)=\dim L-\dim R=2,
 \qquad \dim\pi_{ih}(R)\le2.
\]
Suppose $\dim\pi_i(R)=2$. For every $h\ne i$,
$2=\dim\pi_i(R)\le\dim\pi_{ih}(R)\le2$. Hence the inclusion
$\ker(\pi_{ih}|_R)\subseteq\ker(\pi_i|_R)$ is an equality, since
both kernels have dimension $\dim R-2$. Consequently
\[
 \ker(\pi_i|_R)=\bigcap_{h\ne i}\ker(\pi_{ih}|_R)=\{\vct0\},
\]
which implies $\dim R=\dim\pi_i(R)=2$, a contradiction.
Thus $\dim P_i\le1$ for every $P_i:=\pi_i(R)$.

If $P_i=\{\vct0\}$, take $M_i=I_2$. Otherwise write
$P_i=\operatorname{span}\{\vct w_i\}$ with
$\vct w_i=(a_i,b_i)^{\mathsf T}\ne\vct0$, and set
\[
 M_i=
 \begin{cases}
  \begin{pmatrix}b_i&-a_i\\a_i^{-1}&0\end{pmatrix},&a_i\ne0,\\[3pt]
  \begin{pmatrix}b_i&0\\0&b_i^{-1}\end{pmatrix},&a_i=0.
 \end{cases}
\]
In both cases $\det M_i=1$ and $M_i\vct w_i=(0,1)^{\mathsf T}$.
For $J=\left(\begin{smallmatrix}0&1\\-1&0\end{smallmatrix}\right)$,
$M_i^{\mathsf T}JM_i=J$. Therefore
$T=\operatorname{diag}(M_1,\ldots,M_n)$ is symplectic and
\[
 T(R)\subseteq\{(0,z_1,\ldots,0,z_n):z_i\in\F_q\},
 \qquad \wt(T\vct v)=\wt(\vct v)\quad(\vct v\in\F_q^{2n}).
\]
The weight identity follows because every $M_i$ is invertible.
Thus $T$ gives the required local normalization.
\end{proof}

The lemma also applies to binary targets with $k=1$, $d\ge n-1$
and $j>2$. Adding all radical-$X$ zero constraints to the seven
binary exclusions gives seven further UNSAT instances with checked
DRAT proofs. A feasible $[[6,1,5;2]]_2$ control remains SAT and its
decoded generators pass independent parameter checks. The artifact
preserves these normalized certificates alongside the original ones.

For $[[10,1,9;4]]_3$, $j=5$. Row reduction preserves the zero
$X$ components, so the normalized encoding retains every possible code.
Adding $jn=50$ unit clauses gives $33{,}033$ variables and $476{,}271$ clauses.
The lemma establishes completeness; DRAT certifies unsatisfiability.
As a positive control, the same normalized encoder is SAT when fixed to
the archived $[[10,1,9;5]]_3$ code. The \artifactlink{} supplies the CNF,
proof, normalization metadata and checks of the added clauses.

For certificate verification,\label{app:sat-resources} the original CNFs
are summarized in Table~\ref{tab:sat-resources}. Solver time
measures a fresh UNSAT decision; for $[[10,1,9;5]]_2$, it also includes
proof generation. The check time uses DRAT-trim. Peak is the larger solver/checker resident set.
Concurrent workloads preclude isolated timing comparisons. Execution
conditions and logs are in the \artifactlink{}.

\begin{table}[H]
\centering
\caption{Original SAT refutation resources.}\label{tab:sat-resources}
\small
\setlength{\tabcolsep}{4pt}
\begin{tabular}{@{}lrrrrrr@{}}
\toprule
\multicolumn{1}{c}{\multirow{2}{*}{Excluded code}} & \multicolumn{1}{c}{\multirow{2}{*}{Variables}} & \multicolumn{1}{c}{\multirow{2}{*}{Clauses}} & \multicolumn{1}{c}{Proof} & \multicolumn{1}{c}{Solver} & \multicolumn{1}{c}{Check} & \multicolumn{1}{c}{Peak} \\
 & & & \multicolumn{1}{c}{(MiB)} & \multicolumn{1}{c}{(s)} & \multicolumn{1}{c}{(s)} & \multicolumn{1}{c}{(MiB)} \\
\midrule
$[[5,2,4;3]]_2$ & 414 & 1,275 & 0.03 & 0.02 & 0.46 & 48.2 \\
$[[6,1,5;1]]_2$ & 1,881 & 6,425 & 22.86 & 1.88 & 1.88 & 68.1 \\
$[[6,2,5;4]]_2$ & 498 & 1,540 & 0.04 & 0.02 & 0.04 & 48.1 \\
$[[7,1,6;3]]_2$ & 1,110 & 3,666 & 7.66 & 0.61 & 0.51 & 57.2 \\
$[[8,1,7;3]]_2$ & 2,513 & 8,609 & 172.27 & 16.94 & 13.58 & 195.0 \\
$[[9,1,8;5]]_2$ & 1,430 & 4,734 & 11.86 & 0.89 & 0.73 & 60.2 \\
$[[10,1,9;5]]_2$ & 3,145 & 10,793 & 234.91 & 66.33 & 58.76 & 491.0 \\
$[[11,1,10;7]]_2$ & 1,750 & 5,802 & 16.58 & 1.23 & 0.89 & 63.1 \\
$[[13,1,12;8]]_2$ & 4,093 & 14,069 & 43.48 & 3.46 & 2.48 & 81.3 \\
$[[10,1,9;4]]_3$ & 33,033 & 476,271 & 28.93 & 11.90 & 62.25 & 239.9 \\
\bottomrule
\end{tabular}
\end{table}


\begingroup
\raggedbottom
\subsection{Researcher proposals and task specifications}\label{app:prompts}
These structured summaries separate researcher input from the task given
to the model. Shaded rows mark researcher proposals or findings. Full
specifications are preserved verbatim in the \artifactlink{}.
The controlled study is separate from the discovery campaigns.

\newcommand{\taskheading}[2]{%
  \par\addvspace{8pt}\noindent
  {\small\rmfamily\bfseries\color{achieve}#1}\hfill
  {\footnotesize\rmfamily\color{black!60}#2}%
  \par\nobreak\vspace{3pt}{\color{achieve!45}\hrule height .5pt}%
  \nobreak\vspace{4pt}}
\newcommand{\taskrow}[3]{%
  \begin{tcolorbox}[enhanced,colback=#1,frame hidden,boxrule=0pt,
    arc=0pt,left=6pt,right=6pt,top=3pt,bottom=3pt,
    before skip=0pt,after skip=2pt]
  \small\rmfamily
  \begin{tabularx}{\linewidth}{@{}>{\raggedright\arraybackslash\hyphenpenalty=10000\exhyphenpenalty=10000}p{2.55cm}@{\hspace{8pt}}X@{}}
  {\footnotesize\rmfamily\bfseries\color{black!65}#2} & #3
  \end{tabularx}
  \end{tcolorbox}}
\newcommand{\tasksource}[1]{%
  \par\nobreak\noindent{\scriptsize\rmfamily\color{black!55}#1}\par}


\begin{minipage}{\linewidth}
\taskheading{Discovery campaigns 1--2}{Initial search}
\taskrow{achieve!10}{Researcher proposal}{Explore cyclic qubit shifts, constructive warm starts, $\mathrm{GF}(4)$-additive constructions and coupled generator moves. Reuse structure between targets sharing $(n,c)$.}
\taskrow{white}{Task to the model}{Construct binary EAQECCs for 68 open table entries. Each target $[[n,k,d;c]]_2$ has listed lower distance bound $d-1$ and upper bound $d$.}
\taskrow{white}{Exact feedback}{Match the requested $(k,c)$ and reach distance $d$. The evaluator checks the symplectic constraints and reports the number of nontrivial logical operators below the target distance.}
\taskrow{white}{Knowledge supplied}{Stabilizer representation, evaluator access and construction directions. The family formula derived later was not supplied.}
\tasksource{Source: archived task specifications for campaigns 1 and 2.}
\end{minipage}
\par


\begin{minipage}{\linewidth}
\taskheading{Discovery campaigns 3--4}{After generalization}
\taskrow{achieve!10}{Researcher result}{The recorded qubit constructions $[[7,1,6;4]]_2$, $[[9,1,8;6]]_2$ and $[[11,1,10;8]]_2$ were generalized by researchers to the odd-length family $[[n,1,n-1;n-3]]_2$.}
\taskrow{achieve!10}{Researcher proposal}{Develop further constructions using other shift steps, quasi-cyclic orbits, multiple radical blocks, $\mathrm{GF}(4)$-additive constructions, and extension, shortening or puncturing of the known family.}
\taskrow{white}{Task to the model}{Search 329 qubit parameter targets with $5\le n\le13$, including listed gaps and previously unlisted parameter cells. The task now explicitly includes the discovered qubit family.}
\taskrow{white}{Exact feedback}{Use the same parameter and distance checks. The specification also permits constructing the smaller dual space and converting it to stabilizer generators for evaluation.}
\tasksource{Source: archived task specifications for campaigns 3 and 4.}
\end{minipage}
\par


\begin{minipage}{\linewidth}
\taskheading{Controlled study}{Shared proposal for both policies}
\taskrow{achieve!10}{Researcher proposal}{Maintain and mutate a basis of $L=S^\perp$ throughout search. Convert each candidate $L$ to $S$ for exact evaluation. For the training targets, $L$ has only four basis vectors.}
\taskrow{white}{Task to the model}{Implement \texttt{search(target, max\_evals)} under a separate quota for each target and seed. Independent proposals and iterative evolution receive the same research suggestion.}
\taskrow{white}{Information withheld}{No successful code matrix, family formula or working $L$-side search program is supplied. There is no controller-provided initialization.}
\tasksource{Source: archived controlled-study task; this is separate from the discovery campaigns.}
\end{minipage}
\par


The controlled-study protocol and calibration are specified in
Appendix~\ref{app:hitl-protocol}.


\endgroup

\section{Research-guided ablation details}\label{app:ablations}
\subsection{Shared research guidance and evaluation calibration}
\label{app:hitl-protocol}
The controlled study uses Gemini 3.5 Flash through Vertex
\texttt{generateContent}, with a local program-development controller.
It is distinct from the managed AlphaEvolve discovery campaigns.
The two policies receive identical instructions to maintain and mutate
$L=S^\perp$ throughout search and convert candidate bases to $S$ using
binary nullspace and symplectic-swap helpers. The initial code is the
same $S$-side annealing callback in both policies. There is no protected
initializer. Programs may change representations, objectives, mutation
operators and restart schedules; the working $L$-side calibration code,
archived codes and family formula are not provided to either policy.

Before model generation, four calibration seeds compare the initial
callback with an interface-adapted, researcher-assisted $L$-side control
at up to $10^5$ evaluator calls. The controls depend on the call limit
only through the remaining allowance, so their first-hit records give
success curves at 1,000, 3,000, 10,000, 30,000 and 100,000 calls.
At $n=9$, the rule selects the smallest grid budget with initial
success between 0.1 and 0.5, control success at least 0.5 and a difference
of at least 0.25. A second cell may have zero initial success if the
control succeeds in at least half the executions. The selected budget
is 10,000: initial/control success is 1/4 versus 2/4 at $n=9$, and
0/4 versus 2/4 at $n=11$. Calibration is a test of whether the evaluator
can reward a useful edit, not a measurement of evolutionary contribution.
Independently verified codes establish feasibility at all training and transfer
parameters before generation.

\subsection{Generation, interruption and selection}\label{app:hitl-selection}
Four blocks each plan 12 proposals per policy. Independent proposals
receive the fixed initial source and its fixed training feedback.
Iterative evolution receives the best-so-far source and the latest
attempt, including rejected attempts, with their training feedback.
Training uses two fixed execution seeds at $n=9,11$. Each proposal is
ranked lexicographically by mean normalized success-curve area, success
fraction and clipped distance ratio; earlier proposals break remaining
ties. For first-hit call $h$ under allowance $B$, the primary utility is
\[
 u(h;B)=\begin{cases}1-h/B,&\text{if the target is solved by call }B,\\
 0,&\text{otherwise.}\end{cases}
\]
This equals the normalized integral of the success indicator over the
available evaluator-call budget. The initial program remains eligible
for selection. Proposal counts match; token and execution costs need not.

The planned 12-proposal runs received 58 model responses before
unavailable responses interrupted generation. Before any validation or
test observations, a completion-only rule retained the longest fully
received prefix shared by every policy and block: seven proposals each,
56 total. The two additional responses and all request records are
preserved separately. Earlier exploratory studies informed calibration
and guidance; their measurements are not pooled into this comparison.

For the included 28 requests per policy, recorded input/output/reasoning
token totals are 50,344/22,891/22,780 for independent proposals and
86,501/26,145/32,128 for iterative evolution. Mean request latency is
8.77\,s and 11.93\,s, respectively. The additional source and feedback
in iterative prompts therefore incur greater model usage despite equal
proposal counts. Receipt-level accounting is preserved in the artifact; service latency
is separate from search time.

After generation ends, four new validation seeds at the training
parameters select among each policy/block's three training-best distinct
programs, using the same metric order and then training rank. The eight
selected programs are distinct. Test uses eight further seeds at each
training parameter and 10,000 calls per execution. The separate length
transfer uses $n=13,15$, four seeds and 30,000 calls. Transfer measurements
never enter selection. All seed lists, requests, source hashes, parent
relationships, candidate outcomes and selections are retained in
\texttt{experiments/hitl\_ablation/}.

\subsection{Complete policy outcomes}\label{app:hitl-outcomes}
Table~\ref{tab:hitl-blocks} reports the original test cohort and the
32-seed repeat of the same validation-selected programs. Displayed values
are rounded; differences use unrounded values. The new seeds
were fixed before execution and did not change any program or selection.
The original paired mean curve-area difference is $+0.1188$ (two positive
blocks), versus $+0.0044$ in the repeat (one positive block). Descriptive
two-sided sign-flip values over the four block differences are 0.50 and
1.00. These calculations use four generation blocks; execution seeds
within one program quantify its run-to-run variability. Excluding the
block affected by later unavailable responses gives an original mean
difference of $+0.0819$. All included proposals precede those responses.

\begin{table}[!htbp]
\centering\small
\begin{tabular}{crrrrrr}
\toprule
& \multicolumn{2}{c}{Original /16} & \multicolumn{2}{c}{Repeat /64}
& \multicolumn{2}{c}{Curve-area difference}\\
\cmidrule(lr){2-3}\cmidrule(lr){4-5}\cmidrule(lr){6-7}
Block & Independent & Iterative & Independent & Iterative & Original & Repeat\\
\midrule
1 & 4 & 3 & 38 & 7 & $-0.114$ & $-0.201$ \\
2 & 4 & 10 & 27 & 26 & $+0.229$ & $-0.014$ \\
3 & 5 & 16 & 34 & 49 & $+0.395$ & $+0.287$ \\
4 & 8 & 9 & 36 & 33 & $-0.036$ & $-0.055$ \\
\midrule
Total/mean & 21/64 & 38/64 & 135/256 & 115/256 & $+0.119$ & $+0.004$\\
\bottomrule
\end{tabular}
\caption{Test outcomes by generation block.}\label{tab:hitl-blocks}
\end{table}

Iterative selections are proposals 2, 1, 3 and 1 in blocks 1--4,
respectively (proposal numbers start at one). Thus two selections are
first proposals and two incorporate feedback from earlier generated
programs. The classical $L$-side control uses the shared proposed
representation with the preserved annealing objective, moves and
restart policy; its source is unchanged from calibration.

The repeat covers 32 new seeds at each of $n=9,11,13,15$, at 10,000
calls for the first two lengths and 30,000 for the latter two. It
includes all eight selected programs, both fixed controls, and the
first two programs in the block-3 iterative lineage. Together with
24 executions of the classical control on the original test and
longer-code seeds, this gives 1,560 executions and 545 independently
verified successful code records, with no process errors. All records,
program hashes and the frozen protocol are in
\texttt{experiments/hitl\_followup/}.

At $n=9,11$, the repeat gives 96/128 and 39/128 successes for independent
proposals, versus 81/128 and 34/128 for iterative evolution. The classical
control gives 26/32 and 10/32, and the initial program 1/32 and 0/32.
Figure~\ref{fig:hitl-followup} shows the longer-code curves; the artifact
also retains the complete four-length plot.

\begin{figure}[!htbp]
\centering
\resizebox{\linewidth}{!}{
\definecolor{hfinitial}{HTML}{A37545}
\definecolor{hfind}{HTML}{78818E}
\definecolor{hfevo}{HTML}{00858B}
\definecolor{hfcontrol}{HTML}{B55D32}
\begin{tikzpicture}[x=1cm,y=1cm,font=\scriptsize]
\path[use as bounding box] (-1,-1.70) rectangle (13.3,3.7);
\begin{scope}[xshift=0cm,yshift=0cm]
\node[font=\scriptsize\bfseries] at (2.825,3.2) {(a) $n=13,\ d=12$};
\draw[black!12,line width=.35pt] (0,0.0)--(5.65,0.0);
\node[anchor=east] at (-.12,0.0) {0};
\draw[black!12,line width=.35pt] (0,0.7)--(5.65,0.7);
\node[anchor=east] at (-.12,0.7) {25};
\draw[black!12,line width=.35pt] (0,1.4)--(5.65,1.4);
\node[anchor=east] at (-.12,1.4) {50};
\draw[black!12,line width=.35pt] (0,2.1)--(5.65,2.1);
\node[anchor=east] at (-.12,2.1) {75};
\draw[black!12,line width=.35pt] (0,2.8)--(5.65,2.8);
\node[anchor=east] at (-.12,2.8) {100};
\draw[black!40] (0,2.8)--(0,0)--(5.65,0);
\node[rotate=90] at (-.72,1.4) {Success rate (\%)};
\node at (2.825,-.60) {Exact evaluations};
\node[anchor=north] at (0.0,-.10) {0};
\node[anchor=north] at (1.8833333333333333,-.10) {10k};
\node[anchor=north] at (3.7666666666666666,-.10) {20k};
\node[anchor=north] at (5.65,-.10) {30k};
\draw[hfind,line width=1.0pt] (0.00000,0.00000) -- (1.07237,0.00000) -- (1.07237,0.02187) -- (1.30138,0.02187) -- (1.30138,0.04375) -- (1.46787,0.04375) -- (1.46787,0.06562) -- (1.76525,0.06562) -- (1.76525,0.08750) -- (3.14385,0.08750) -- (3.14385,0.10938) -- (3.45422,0.10938) -- (3.45422,0.13125) -- (3.87496,0.13125) -- (3.87496,0.15312) -- (3.91488,0.15312) -- (3.91488,0.17500) -- (4.14107,0.17500) -- (4.14107,0.19687) -- (4.55823,0.19687) -- (4.55823,0.21875) -- (4.55899,0.21875) -- (4.55899,0.24062) -- (4.68121,0.24062) -- (4.68121,0.26250) -- (4.70739,0.26250) -- (4.70739,0.28437) -- (4.73602,0.28437) -- (4.73602,0.30625) -- (5.46751,0.30625) -- (5.46751,0.32812) -- (5.50442,0.32812) -- (5.50442,0.35000) -- (5.61007,0.35000) -- (5.61007,0.37187) -- (5.65000,0.37187);
\draw[hfevo,line width=1.0pt] (0.00000,0.00000) -- (0.29964,0.00000) -- (0.29964,0.02187) -- (0.80776,0.02187) -- (0.80776,0.04375) -- (0.82377,0.04375) -- (0.82377,0.06562) -- (0.82528,0.06562) -- (0.82528,0.08750) -- (0.98404,0.08750) -- (0.98404,0.10938) -- (1.04808,0.10938) -- (1.04808,0.13125) -- (1.14883,0.13125) -- (1.14883,0.15312) -- (1.51194,0.15312) -- (1.51194,0.17500) -- (1.57334,0.17500) -- (1.57334,0.19687) -- (1.59236,0.19687) -- (1.59236,0.21875) -- (1.62362,0.21875) -- (1.62362,0.24062) -- (1.72702,0.24062) -- (1.72702,0.26250) -- (1.87900,0.26250) -- (1.87900,0.28437) -- (1.90763,0.28437) -- (1.90763,0.30625) -- (1.94153,0.30625) -- (1.94153,0.32812) -- (2.08598,0.32812) -- (2.08598,0.35000) -- (2.12120,0.35000) -- (2.12120,0.37187) -- (2.16979,0.37187) -- (2.16979,0.39375) -- (2.17299,0.39375) -- (2.17299,0.41562) -- (2.29691,0.41562) -- (2.29691,0.43750) -- (2.42310,0.43750) -- (2.42310,0.45937) -- (2.46867,0.45937) -- (2.46867,0.48125) -- (2.49353,0.48125) -- (2.49353,0.50312) -- (2.75833,0.50312) -- (2.75833,0.52500) -- (2.91069,0.52500) -- (2.91069,0.54688) -- (3.25986,0.54688) -- (3.25986,0.56875) -- (3.42955,0.56875) -- (3.42955,0.59062) -- (3.44951,0.59062) -- (3.44951,0.61250) -- (3.55121,0.61250) -- (3.55121,0.63437) -- (3.91074,0.63437) -- (3.91074,0.65625) -- (4.00623,0.65625) -- (4.00623,0.67812) -- (4.17780,0.67812) -- (4.17780,0.70000) -- (4.40399,0.70000) -- (4.40399,0.72187) -- (4.54147,0.72187) -- (4.54147,0.74375) -- (4.66728,0.74375) -- (4.66728,0.76562) -- (5.04790,0.76562) -- (5.04790,0.78750) -- (5.28840,0.78750) -- (5.28840,0.80937) -- (5.45507,0.80937) -- (5.45507,0.83125) -- (5.49594,0.83125) -- (5.49594,0.85312) -- (5.65000,0.85312);
\draw[hfcontrol,line width=1.0pt,dash pattern=on 4pt off 1.5pt] (0.00000,0.00000) -- (1.46787,0.00000) -- (1.46787,0.08750) -- (4.70739,0.08750) -- (4.70739,0.17500) -- (5.61007,0.17500) -- (5.61007,0.26250) -- (5.65000,0.26250);
\draw[hfinitial,line width=1.0pt,dash pattern=on 3pt off 2pt] (0.00000,0.00000) -- (5.65000,0.00000);
\end{scope}
\begin{scope}[xshift=7.2cm,yshift=0cm]
\node[font=\scriptsize\bfseries] at (2.825,3.2) {(b) $n=15,\ d=14$};
\draw[black!12,line width=.35pt] (0,0.0)--(5.65,0.0);
\node[anchor=east] at (-.12,0.0) {0};
\draw[black!12,line width=.35pt] (0,0.7)--(5.65,0.7);
\node[anchor=east] at (-.12,0.7) {25};
\draw[black!12,line width=.35pt] (0,1.4)--(5.65,1.4);
\node[anchor=east] at (-.12,1.4) {50};
\draw[black!12,line width=.35pt] (0,2.1)--(5.65,2.1);
\node[anchor=east] at (-.12,2.1) {75};
\draw[black!12,line width=.35pt] (0,2.8)--(5.65,2.8);
\node[anchor=east] at (-.12,2.8) {100};
\draw[black!40] (0,2.8)--(0,0)--(5.65,0);
\node[rotate=90] at (-.72,1.4) {Success rate (\%)};
\node at (2.825,-.60) {Exact evaluations};
\node[anchor=north] at (0.0,-.10) {0};
\node[anchor=north] at (1.8833333333333333,-.10) {10k};
\node[anchor=north] at (3.7666666666666666,-.10) {20k};
\node[anchor=north] at (5.65,-.10) {30k};
\draw[hfind,line width=1.0pt] (0.00000,0.00000) -- (3.09639,0.00000) -- (3.09639,0.02187) -- (5.65000,0.02187);
\draw[hfevo,line width=1.0pt] (0.00000,0.00000) -- (0.57065,0.00000) -- (0.57065,0.02187) -- (0.63337,0.02187) -- (0.63337,0.04375) -- (0.66953,0.04375) -- (0.66953,0.06562) -- (1.17897,0.06562) -- (1.17897,0.08750) -- (1.39103,0.08750) -- (1.39103,0.10938) -- (1.43359,0.10938) -- (1.43359,0.13125) -- (1.44471,0.13125) -- (1.44471,0.15312) -- (1.66204,0.15312) -- (1.66204,0.17500) -- (1.68408,0.17500) -- (1.68408,0.19687) -- (2.08447,0.19687) -- (2.08447,0.21875) -- (2.66850,0.21875) -- (2.66850,0.24062) -- (2.67716,0.24062) -- (2.67716,0.26250) -- (3.29903,0.26250) -- (3.29903,0.28437) -- (4.24880,0.28437) -- (4.24880,0.30625) -- (4.50550,0.30625) -- (4.50550,0.32812) -- (4.68366,0.32812) -- (4.68366,0.35000) -- (4.77067,0.35000) -- (4.77067,0.37187) -- (5.29876,0.37187) -- (5.29876,0.39375) -- (5.34471,0.39375) -- (5.34471,0.41562) -- (5.40328,0.41562) -- (5.40328,0.43750) -- (5.65000,0.43750);
\draw[hfcontrol,line width=1.0pt,dash pattern=on 4pt off 1.5pt] (0.00000,0.00000) -- (5.65000,0.00000);
\draw[hfinitial,line width=1.0pt,dash pattern=on 3pt off 2pt] (0.00000,0.00000) -- (5.65000,0.00000);
\end{scope}
\draw[hfinitial,line width=1pt,dash pattern=on 3pt off 2pt] (1,-1.1)--(1.55,-1.1);
\node[anchor=west] at (1.65,-1.1) {Initial program};
\draw[hfcontrol,line width=1pt,dash pattern=on 3pt off 2pt] (7.1,-1.1)--(7.6499999999999995,-1.1);
\node[anchor=west] at (7.75,-1.1) {Classical L-side control};
\draw[hfind,line width=1pt] (1,-1.53)--(1.55,-1.53);
\node[anchor=west] at (1.65,-1.53) {Independent proposals};
\draw[hfevo,line width=1pt] (7.1,-1.53)--(7.6499999999999995,-1.53);
\node[anchor=west] at (7.75,-1.53) {Iterative evolution};
\end{tikzpicture}

}
\caption{Longer-code replay on fresh execution seeds.}
\label{fig:hitl-followup}
\end{figure}
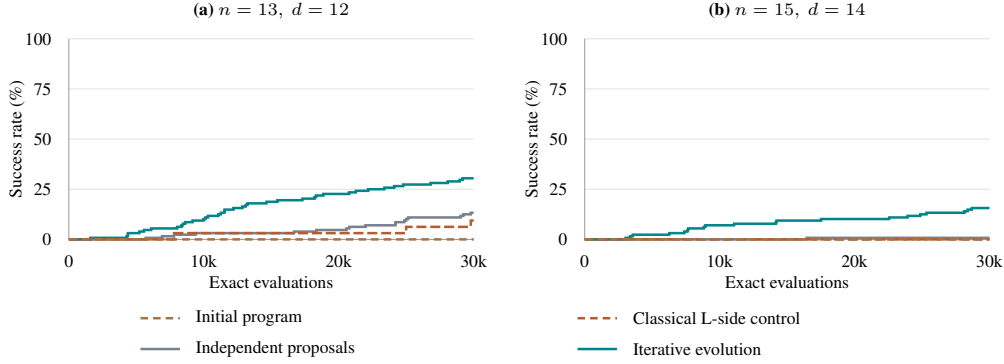


\FloatBarrier

Helper computations remain outside the evaluator-call count but inside
the common resource limits. Each target/seed runs in a separate sandbox
with 120/125\,s soft/hard CPU limits, 135\,s wall time and 1\,GiB resident
memory. Every explicit or returned-candidate verification consumes the
same call quota. Execution times and model request latencies are
recorded separately in the artifact.

The artifact audit reconstructs the original proposal inputs and parent
relationships, checks validation selections and independently verifies
saved code records. The complete inputs and replay commands are available
in the \artifactlink{}.

\subsection{Prospective generation replication}\label{app:replication}
The new cohort fixes eight paired blocks and seven proposal slots per
policy before generation. Model, shared task, starting program,
representation helpers, ranking and validation selection match the
original study. New model seeds and disjoint training, validation,
test and longer-code execution seeds separate the cohorts. Each block
uses two training seeds, four validation seeds and 32 held-out seeds
per length. The test budget is 10,000 calls at $n=9,11$; longer-code
performance at 30,000 calls is secondary.

Of 112 allocated slots, 111 return programs: 56 independent and 55
iterative. One response in block 8 is unavailable. Before validation
or testing, the recorded rule retained all eight blocks, using received
programs without replacing the missing proposal. Subsequent iterative
requests use the most recent received source and feedback. The seven
fully received paired blocks provide a sensitivity analysis. No
program is reselected using test or longer-code outcomes.

Table~\ref{tab:replication-blocks} reports all blocks; each longer-length
entry has 32 executions per policy and block. The mean primary
curve-area difference is $+0.0987$; five differences are positive and
the descriptive two-sided sign-flip value is $p=0.3281$. On the seven
complete blocks, the mean difference is $+0.1504$. These are eight
paired generation observations, not 512 independent generation trials.

\begin{table}[!htbp]
\centering\small
\setlength{\tabcolsep}{5pt}
\begin{tabular}{crrrrr}
\toprule
Block & Independent & Iterative & $\Delta$ area & $n=13$ & $n=15$\\
 & \multicolumn{2}{c}{Test successes /64} & & \multicolumn{2}{c}{Independent / iterative}\\
\midrule
1 & 25 & 19 & -0.063 & 3 / 5 & 0 / 2\\
2 & 25 & 30 & +0.007 & 3 / 10 & 0 / 5\\
3 & 24 & 35 & -0.001 & 2 / 15 & 1 / 4\\
4 & 39 & 45 & +0.021 & 17 / 16 & 5 / 7\\
5 & 26 & 64 & +0.578 & 6 / 32 & 2 / 32\\
6 & 34 & 40 & +0.095 & 15 / 8 & 0 / 3\\
7 & 25 & 62 & +0.417 & 3 / 32 & 0 / 28\\
8 & 25 & 2 & -0.264 & 1 / 0 & 0 / 0\\
\midrule
Total/mean & 223/512 & 297/512 & $+0.099$ & 50 / 118 & 8 / 81\\
\bottomrule
\end{tabular}
\caption{New generation blocks.}\label{tab:replication-blocks}
\end{table}


All 16 selected sources are distinct. Source inspection finds sustained
search in $L$, with exact distance feedback obtained through the common
charged evaluator. Rank checks and basis conversions do not compute
distance privately. One program performs uncharged basis-row additions
that leave the subspace unchanged; other proposals evaluate their
candidate subspaces. Reads of the incumbent before search return no
matrix because controller initialization is disabled. The artifact audit
reconstructs all 111 prompts and parent choices, reproduces validation
selection, checks evaluation quotas and independently verifies saved codes.

The received independent requests use 100,632 input, 41,801 output and
58,067 reported reasoning tokens; iterative requests use 165,819,
49,296 and 64,426, respectively. Equal proposal allocations therefore
do not imply equal token costs. Full transport records and the
missing-response rule are preserved in the artifact.

For Figure~\ref{fig:ablation-comparison}, the original initial program
and classical $L$-side control are also run on these execution seeds.
This fixed-control replay was added after the new policy results were
available; neither control was modified, tuned or selected using them.
Each control has 32 runs per length, and is separate from the eight
paired policy observations. These runs contextualize the success curves;
the prespecified policy endpoint and its analysis remain unchanged.


\subsection{Preserved program and generalization to longer codes}\label{app:hitl-program}
In the original four-block cohort, the program selected in block 3 is
\texttt{b02\_evolution\_g02}, whose preserved parent is
\texttt{b02\_evolution\_g01}; generation indices start at zero.
The parent already searches $L$ and penalizes distance shortfall.
The child increases the ebit-mismatch weight from 5,000 to 10,000 and
the offending weight from 1 to 100. It also mixes row replacement,
Pauli perturbations and basis-row mixing, scales mutation size with
length, and restores a saved local best with a capped temperature
adjustment after stagnation. Tinted expressions in Algorithm~\ref{alg:ablation-search} mark these edits
relative to this immediate parent; the inherited $L$-side representation
and distance-shortfall term are not tinted. These changes are observed
in the preserved source, rather than inferred from performance alone. The repeat below measures the combined changes in this final step.

\begin{algorithm}[!t]
\caption{Evolved search in the symplectic dual.}
\label{alg:ablation-search}
\small
\begin{algorithmic}[1]
\Require target $(n,k_t,c_t,d_t)$, exact-evaluation budget $N$
\Statex \(\Phi(R)=\algdelta{\(10^4\)}|R.c-c_t|+\algdelta{\(100\)}R.\mathrm{offending}+10^3\Delta_d(R)\)
  \algside{Increase penalties for entanglement mismatch and low-weight violations.}
\State $r\gets n+k_t-c_t$
\While{budget remains}
  \State $L\gets$ a random independent list of $r$ Pauli vectors
  \State $C\gets\Phi(\textsc{Eval}(L^\perp))$; \algdelta{$(L_*,C_*)\gets(L,C)$}
    \algside{Retain the best state within each restart.}
  \State $\tau\gets\algdelta{\(50\)}$; $f\gets0$; \algdelta{$g\gets0$}
  \While{$f<\algdelta{\(600\)}$ \textbf{and} budget remains}
    \State $\tau\gets\max(\algdelta{\(0.01\)},0.995\tau)$
    \State $L'\gets\algdelta{\(\textsc{MixedMove}(L)\)}$
      \algside{Add row mixing and allow larger Pauli changes at longer lengths.}
    \State \textbf{if} $\operatorname{rank}\,L'\ne r$ \textbf{then} $f\gets f+1$; \textbf{continue}
    \State $C'\gets\Phi(\textsc{Eval}((L')^\perp))$; $\delta\gets C'-C$
    \If{$\delta\le0$ \textbf{or} $U(0,1)<\exp(-\delta/\tau)$}
      \State $(L,C)\gets(L',C')$
      \If{\algdelta{$C<C_*$}}
        \State \algdelta{$(L_*,C_*)\gets(L,C)$; $g\gets0$}
      \Else
        \State \algdelta{$g\gets g+1$}
      \EndIf
      \State \textbf{if} $\delta<0$ \textbf{then} $f\gets0$ \textbf{else} $f\gets f+1$
    \Else
      \State $f\gets f+1$; \algdelta{$g\gets g+1$}
    \EndIf
    \If{\algdelta{$g>80$}}
      \algside{On stagnation, return to the best state and adjust the temperature.}
      \State \algdelta{$(L,C)\gets(L_*,C_*)$; $g\gets0$}
      \State \algdelta{$\tau\gets\min(30,1.5\tau)$}
    \EndIf
  \EndWhile
\EndWhile
\State \Return the evaluator's best verified record, if any
\end{algorithmic}
\vspace{2pt}
{\footnotesize\algdelta{Shading}: changes from the immediate parent; exact definitions in Appendix~\ref{app:hitl-program}.}
\end{algorithm}


The program cost in Algorithm~\ref{alg:ablation-search} is
\[
 \Phi(R)=10^4|R.c-c_t|+100\,R.\mathrm{offending}+10^3\Delta_d(R),
\]
where $\Delta_d(R)=\max(0,d_t-R.d)$, or $d_t$ if distance is unavailable.
If the offending count or entanglement cost is missing, the score is
$10^9$. These definitions match the preserved source.

In Algorithm~\ref{alg:ablation-search}, a random Pauli perturbation has weight
uniformly sampled from $1,\ldots,\max(2,\lfloor n/2\rfloor)-1$.
\textsc{MixedMove} chooses a row uniformly: with probability 0.3 it
replaces the row by such a perturbation, with probability 0.3 it XORs
the perturbation into the row, and otherwise it samples a second row.
If that row is different, it XORs the two rows and, with probability
0.5, adds a Pauli perturbation. A rank-deficient proposal increments
$f$ and skips evaluation. Otherwise, $f$ resets only on an accepted
improvement over the current cost, and increments on all other moves.
The counter $g$ resets on a new restart-local best and increments on
other evaluated moves. At $g>80$, the program restores that best,
sets $g=0$ and applies $\tau\leftarrow\min(30,1.5\tau)$. This adjustment
can lower a temperature still above 30; it is not always a temperature
increase. The source and its budget behavior remain unchanged.

\newpage
The original child achieves 16/16 primary-test successes and 3/4 at
each longer length. Its successful longer-code executions use
3,260--9,755 calls at $n=13$ and 7,125--16,756 at $n=15$.
Table~\ref{tab:hitl-lineage} reports every preserved lineage stage on
the 32 new seeds; the original validation selection is unchanged.
Between the first proposal and parent, the only functional search change
is the distance-shortfall penalty. Between parent and selected child,
several policies change together. This lineage was chosen for diagnosis
after its original performance, so the table describes that program's
development separately from the four-block policy comparison.

\begin{table}[!htbp]
\centering\small
\begin{tabular}{lrrrr}
\toprule
Program & $n=9$ & $n=11$ & $n=13$ & $n=15$\\
\midrule
First proposal & 27/32 & 19/32 & 13/32 & 3/32\\
Parent: add distance penalty & 28/32 & 23/32 & 29/32 & 28/32\\
Validation-selected child & 30/32 & 19/32 & 30/32 & 20/32\\
\bottomrule
\end{tabular}
\caption{Preserved lineage on fresh execution seeds.}\label{tab:hitl-lineage}
\end{table}

All eight selected programs implement sustained $L$-side repair. Each
passes candidate matrices to the common evaluator throughout search;
none substitutes a private distance test that filters solutions before
their first charged evaluation.


\FloatBarrier

\clearpage
\section{Lean verification of Theorem 1}\label{app:lean}
The formal artifact verifies the algebraic construction underlying
Theorem~\ref{thm:unified}. Its shared definitions import Mathlib; the
construction proof is a separate module. The challenge imports only
those definitions, and the submitted solution imports the construction
proof. The sources and pinned project are in the \artifactlink{}.

\lstdefinestyle{leancode}{basicstyle=\ttfamily\footnotesize,
  keywordstyle=\bfseries,backgroundcolor=\color{achieve!5},
  frame=single,rulecolor=\color{achieve!35},framerule=.4pt,
  xleftmargin=4pt,xrightmargin=4pt,framexleftmargin=4pt,
  breaklines=true,columns=fullflexible,keepspaces=true,
  showstringspaces=false,escapeinside={(*@}{@*)},
  morekeywords={theorem,by,exact,let,have,rcases,with,obtain,if,then,else}}

\subsection{Exact goal and mathematical interpretation}\label{app:lean-theorem}
Use the field $\F_q$ and enumeration $u_1,\ldots,u_q$ from
Appendix~\ref{app:theorem}. Lean denotes this field by $\mathbb F$ and
its cardinality by \texttt{Fintype.card} $\mathbb F$, which equals $q$.
The natural number \texttt{tail} is $t=n-2q>0$; hence $n=2q+t$.
The disjunction below specifies the same $q\ge3$ and odd-length binary
cases as Theorem~\ref{thm:unified}. Within namespace \texttt{EAQECCFamily.Comparator}, the
submitted declaration is:

\begin{lstlisting}[style=leancode]
theorem theorem1_exact_goal
    {(*@$\mathbb{F}$@*) : Type u} [Field (*@$\mathbb{F}$@*)] [Fintype (*@$\mathbb{F}$@*)] [DecidableEq (*@$\mathbb{F}$@*)]
    {tail : (*@$\mathbb{N}$@*)} [NeZero tail]
    (hcase : 3 (*@$\le$@*) Fintype.card (*@$\mathbb{F}$@*) (*@$\lor$@*)
      (Fintype.card (*@$\mathbb{F}$@*) = 2 (*@$\land$@*) Odd (2 * Fintype.card (*@$\mathbb{F}$@*) + tail))) :
    (*@$\exists$@*) z(*@${}_2$@*) lam : (*@$\mathbb{F}$@*), lam (*@$\ne$@*) 0 (*@$\land$@*) 1 + z(*@${}_2$@*) (*@$\ne$@*) 0 (*@$\land$@*)
      Nonempty (NormalizerSideCertificate (*@$\mathbb{F}$@*) tail z(*@${}_2$@*) lam) := by
  exact theorem1_exact hcase
\end{lstlisting}


Lean represents the radical coefficients by the function space
\[
 \mathcal T=\{f:\F_q\to\F_q\}.
\]
A function $f$ assigns one scalar to each field element; $f$ need not be linear. Addition and scalar multiplication are pointwise,
$(f+g)(u)=f(u)+g(u)$ and $(a f)(u)=a f(u)$ for $a\in\F_q$.
The enumeration $u_1,\ldots,u_q$ identifies $f$ with the vector
$\vct\tau=(f(u_1),\ldots,f(u_q))\in\F_q^q$.
This evaluation map is a linear isomorphism, so $\dim_{\F_q}\mathcal T=q$.
Thus Lean's \texttt{ClosedCoeff} type is
$\mathcal C=\F_q^2\oplus\mathcal T\cong\F_q^{q+2}$.
Write $\vct\eta=(\alpha,\beta,f)$ and $\tau_i=f(u_i)$, matching the
coefficients of $\vct a,\vct b,\vct\rho_i$ in Appendix~\ref{app:theorem}.
Relabeling the paired coordinates by $i=1,\ldots,q$, the linear map
$\Phi=\texttt{closedWord}$ is
\[
\begin{aligned}
 \Phi(\vct{\eta})_{(i,0)}&=(\alpha+\beta u_i,\,\beta+\tau_i),\\
 \Phi(\vct{\eta})_{(i,1)}&=(\alpha+\beta u_i,\,\beta z_2-\tau_i),\\
 \Phi(\vct{\eta})_{\mathrm{tail},j}&=(\alpha,\,\beta\lambda_j),
 \qquad 0\le j<t,
\end{aligned}
\]
where $\lambda_0=\lambda$ and $\lambda_j=1$ for $j>0$.
Thus $\Phi(\vct{\eta})=\alpha\vct a+\beta\vct b+
\sum_{i=1}^q\tau_i\vct\rho_i$ in the formal coordinate order.
Let $L=\operatorname{im}\Phi$, $\mathcal C_0=
\{(0,0,f):f\in\mathcal T\}$, and
$\delta=\langle\Phi(1,0,0_{\mathcal T}),\Phi(0,1,0_{\mathcal T})\rangle_s$,
where $0_{\mathcal T}$ is the zero function.
The certificate constructed by \texttt{theorem1\_exact} proves:
\begingroup
\renewcommand{\labelenumi}{(\arabic{enumi})}
\begin{enumerate}\setlength{\itemsep}{2pt}
\item \texttt{encodeInjective} and \texttt{deltaNonzero}:
$\Phi$ is injective and $\delta\ne0$.
\item \texttt{gramFormula} and \texttt{radicalIff}: for all
$\vct\eta=(\alpha,\beta,f)$ and
$\vct\eta'=(\alpha',\beta',g)$,
\[
 \langle\Phi(\vct\eta),\Phi(\vct\eta')\rangle_s
 = (\alpha\beta'-\beta\alpha')\delta,
 \qquad \Phi(\vct\eta)\in\operatorname{rad}(L)
 \iff \vct\eta\in\mathcal C_0.
\]
\item \texttt{normalizerFinrank} and \texttt{radicalFinrank}:
$\dim L=q+2$ and $\dim\mathcal C_0=q$.
\item \texttt{distanceLower} and \texttt{distanceAttained}:
\[
 \forall\vct\eta\notin\mathcal C_0,\quad
 \wt(\Phi(\vct\eta))\ge n-1,
 \qquad
 \exists\vct\eta\notin\mathcal C_0,\quad
 \wt(\Phi(\vct\eta))=n-1.
\]
\end{enumerate}
\endgroup
Injectivity identifies $\Phi(\mathcal C_0)$ with the $q$-dimensional
radical of $L$. Taking $S=L^\perp$ gives $\dim S=2n-q-2$.
The standard stabilizer correspondence
(Appendix~\ref{app:symplectic}) therefore yields
$c=(\dim S-q)/2=n-q-1$, $k=n-\dim S+c=1$ and $d=n-1$.
These are the code parameters recorded by the certificate; the
proof-bearing fields above certify their underlying symplectic construction.

For the coordinate convention, Lean's $(u_i,\mathrm{false})$ and
$(u_i,\mathrm{true})$ are our
$(i,0)$ and $(i,1)$. Its tail puts $\lambda$ first; Appendix~\ref{app:theorem}
puts it last. This permutation preserves the symplectic product and weight.
The proof chooses $z_2=0$ in characteristic two and $z_2=1$ otherwise.

Figure~\ref{fig:lean-dependencies} groups the certificate's mathematical
dependencies. The dashed arrow is the standard stabilizer interpretation
used in the paper; it is separate from the Lean-checked algebraic fields.
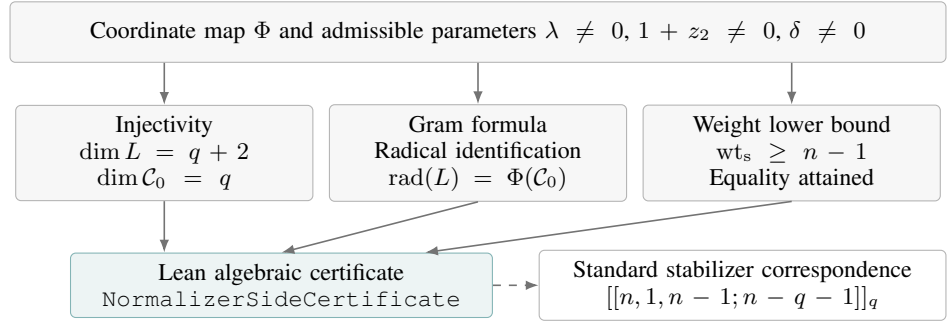
\begin{figure}[!htbp]
\centering
\begin{tikzpicture}[font=\footnotesize,
  box/.style={draw=black!30,rounded corners=2pt,fill=black!3,
    align=center,inner sep=4pt,minimum height=8mm},
  arr/.style={-{Latex[length=1.7mm]},draw=black!55,line width=.6pt}]
\node[box,text width=12.1cm] (map) at (0,0)
  {Coordinate map $\Phi$ and admissible parameters
   $\lambda\ne0$, $1+z_2\ne0$, $\delta\ne0$};
\node[box,text width=3.65cm,minimum height=12mm] (dim) at (-4.15,-1.6)
  {Injectivity\\$\dim L=q+2$\\$\dim\mathcal C_0=q$};
\node[box,text width=3.65cm,minimum height=12mm] (gram) at (0,-1.6)
  {Gram formula\\Radical identification\\$\operatorname{rad}(L)=\Phi(\mathcal C_0)$};
\node[box,text width=3.65cm,minimum height=12mm] (dist) at (4.15,-1.6)
  {Weight lower bound\\$\wt\ge n-1$\\Equality attained};
\node[box,fill=achieve!8,draw=achieve!55,text width=5.3cm] (cert) at (-2.6,-3.35)
  {Lean algebraic certificate\\\texttt{NormalizerSideCertificate}};
\node[box,fill=white,text width=5.1cm] (code) at (3.5,-3.35)
  {Standard stabilizer correspondence\\$[[n,1,n-1;n-q-1]]_q$};
\draw[arr] (map.south -| dim.north)--(dim.north);
\draw[arr] (map.south)--(gram.north);
\draw[arr] (map.south -| dist.north)--(dist.north);
\draw[arr] (dim.south)--(dim.south |- cert.north);
\draw[arr] (gram.south)--(cert.north);
\draw[arr] (dist.south)--([xshift=1.9cm]cert.north);
\draw[arr,dashed] (cert.east)--(code.west);
\end{tikzpicture}
\caption{Dependencies of the algebraic certificate.}
\label{fig:lean-dependencies}
\end{figure}


\FloatBarrier

\subsection{Solution and formal verification}\label{app:lean-replay}
The solution above invokes \texttt{theorem1\_exact}. Its proof body is
reproduced verbatim below. For $q\ge3$, it chooses $\lambda$ outside two
forbidden values. In the binary branch, odd length permits $\lambda=1$.
Both branches invoke \texttt{theorem1\_closed\_form}, which assembles the
certificate from the injectivity, Gram, radical, dimension and distance
lemmas in the same source module.

\begin{lstlisting}[style=leancode]
  let z(*@${}_2$@*) : (*@$\mathbb{F}$@*) := if (2 : (*@$\mathbb{F}$@*)) = 0 then 0 else 1
  have hz(*@${}_2$@*) : 1 + z(*@${}_2$@*) (*@$\ne$@*) 0 := by
    dsimp [z(*@${}_2$@*)]
    split_ifs with htwo
    (*@$\cdot$@*) simp
    (*@$\cdot$@*) simpa [one_add_one_eq_two] using htwo
  rcases hcase with hlarge | (*@$\langle$@*)hcard, hodd(*@$\rangle$@*)
  (*@$\cdot$@*) obtain (*@$\langle$@*)lam, hlam, h(*@$\delta$@*)base(*@$\rangle$@*) := exists_nonzero_add_ne_zero hlarge
        (((2 * Fintype.card (*@$\mathbb{F}$@*) + tail - 1 : (*@$\mathbb{N}$@*)) : (*@$\mathbb{F}$@*)))
    have h(*@$\delta$@*) : closedDelta (tail := tail) z(*@${}_2$@*) lam (*@$\ne$@*) 0 := by
      rw [closedDelta_eq_cast_length_sub_one_add]
      exact h(*@$\delta$@*)base
    exact (*@$\langle$@*)z(*@${}_2$@*), lam, hlam, hz(*@${}_2$@*), theorem1_closed_form z(*@${}_2$@*) lam hlam hz(*@${}_2$@*) h(*@$\delta$@*)(*@$\rangle$@*)
  (*@$\cdot$@*) let lam : (*@$\mathbb{F}$@*) := 1
    have hlam : lam (*@$\ne$@*) 0 := one_ne_zero
    have h(*@$\delta$@*) : closedDelta (tail := tail) z(*@${}_2$@*) lam (*@$\ne$@*) 0 := by
      rw [closedDelta_eq_cast_length_sub_one_add]
      exact binary_odd_delta_ne_zero hcard hodd
    exact (*@$\langle$@*)z(*@${}_2$@*), lam, hlam, hz(*@${}_2$@*), theorem1_closed_form z(*@${}_2$@*) lam hlam hz(*@${}_2$@*) h(*@$\delta$@*)(*@$\rangle$@*)
\end{lstlisting}


The proof and submitted solution use only \texttt{propext},
\texttt{Quot.sound} and \texttt{Classical.choice}. Comparator confirms
that the solution proves the specified goal under this axiom policy,
and Lean's kernel accepts the proof~\citep{LeanComparator}.
The \artifactlink{} provides the pinned project and reproducible
verification records.


\end{document}